\documentclass{article}

\usepackage{arxiv}

\usepackage[english]{babel}

\usepackage{amssymb}
\usepackage{amsmath}
\usepackage{amsthm}
\usepackage{dsfont} %for the symbol of indicator 
\usepackage{graphicx}

\usepackage{mathtools} % amsmath with fixes and additions

\usepackage{thmtools,thm-restate}

\usepackage{cases}
\usepackage{hyperref}

\usepackage{tikz} % nice language for creating drawings and diagrams
\usetikzlibrary{decorations.pathreplacing,calligraphy}
\usetikzlibrary{decorations.pathmorphing}
\usetikzlibrary{calc,trees,positioning,arrows,fit,shapes,calc}

\usetikzlibrary{snakes}

\usepackage{pgfplots}
\usepackage[ruled,vlined,linesnumbered]{algorithm2e}
\usepackage{multirow}

\usepackage{subcaption}
\usepackage{natbib} % has a nice set of citation styles and commands
\newtheorem{definition}{Definition}%[section]
\newtheorem{proposition}{Proposition}%[section]
\newtheorem{theorem}{Theorem}%[section]
\newtheorem{assumption}{Assumption}%[section]
\newtheorem{lemma}{Lemma}%[section]
\newtheorem{example}{Example}%[section]

\definecolor{easyorange}{RGB}{254, 106, 44} 
\definecolor{easyred}{RGB}{211, 93, 110}

\definecolor{TRCA-logique-dis}{RGB}{27,158,119}
\definecolor{TRCA-logique-mixte}{RGB}{217,95,2}
\definecolor{TRCA-PC}{RGB}{117,112,179}
\definecolor{CIRCA}{RGB}{231,41,138}
\definecolor{EasyRCA}{RGB}{102,166,30}
\definecolor{CloudRange}{RGB}{230,171,2}
\definecolor{MicroCause}{RGB}{166,118,29}

\begin{document}
	
	%\title{Root cause analysis through causal discovery in a threshold based monitoring system}
	\title{On the Fly Detection of Root Causes from Observed Data with Application to IT Systems}

	
	\author{\href{https://orcid.org/0000-0003-4695-5059}{\includegraphics[scale=0.06]{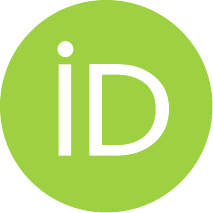}\hspace{1mm}Lei Zan} \\ Univ. Grenoble Alpes, CNRS, \\Grenoble INP, LIG,\\ EasyVista\\
		F38000, Grenoble, France
		\And 
		\href{https://orcid.org/0000-0003-3571-3636}{\includegraphics[scale=0.06]{orcid.pdf}\hspace{1mm}Charles K. Assaad} \\ Sorbonne Université, INSERM,\\ Institut Pierre Louis d'Epidémiologie\\ et de Santé Publique,\\
		F75012, Paris, France
		\And
		\href{https://orcid.org/0000-0002-8360-1834}
		{\includegraphics[scale=0.06]{orcid.pdf}\hspace{1mm}Emilie~Devijver} \\Univ. Grenoble Alpes, CNRS,\\ Grenoble INP, LIG,\\
		F38000, Grenoble, France
		\And 
		\href{https://orcid.org/0000-0002-8858-3233}{\includegraphics[scale=0.06]{orcid.pdf}\hspace{1mm}Eric Gaussier} \\ Univ. Grenoble Alpes, CNRS,\\ Grenoble INP, LIG\\
		F38000, Grenoble, France
		\And 
		\href{https://orcid.org/0009-0004-2715-7363}{\includegraphics[scale=0.06]{orcid.pdf}\hspace{1mm}Ali Aït-Bachir} \\ EasyVista\\
		F38000, Grenoble, France
	}
	
	\date{}
	
	\maketitle
	
	\begin{abstract}
		This paper introduces a new structural causal model tailored for representing threshold-based IT systems and presents a new algorithm designed to rapidly detect root causes of anomalies in such systems. When root causes are not causally related, the method is proven to be correct; while an extension is proposed based on the intervention of an agent to relax this assumption. 
		Our algorithm and its agent-based extension leverage causal discovery from offline data and engage in subgraph traversal when encountering new anomalies in online data. Our extensive experiments demonstrate the superior performance of our methods, even when applied to data generated from alternative structural causal models or real IT monitoring data.  
	\end{abstract}

	%%%%%%%%%%%%%%%%%%%%%%%%%%%%%%%%%%%%%%%%%%%%%%%%%%%%%%%%%%

	\section{Introduction}\label{sec:intro}
	IT monitoring systems are described by metrics, as CPU usage, memory usage, or network traffic, and represented by continuous observational time series. 
	In threshold-based IT monitoring systems, predefined thresholds are used to determine when an anomaly or an alert should be triggered~\citep{Ligus_2013}, where the thresholds are set manually or through algorithms leveraging offline (i.e., historical) data~\citep{dani2015adaptive}.
	%Each metric, such as CPU usage, memory usage, or network traffic, is represented by a continuous observational time series data in the system, and comes with a designated threshold. 
	%These thresholds are set manually or through algorithms leveraging historical data~\citep{dani2015adaptive}. 
	%When a metric surpasses its threshold, an anomaly is triggered, indicating a potential issue. 
	In IT systems with multiple interconnected subsystems, several metrics may go into an anomalous state during an incident. In this context, root cause analysis consists on identifying actionable root causes of the anomalies that can be used to resolve the incident. This process is crucial for mitigating impacts like significant financial losses during system outages~\cite{bajak2021aws}.
	
	Causal graphs can help infer root causes, but obtaining them from experts is often impractical. Causal discovery methods~\citep{Spirtes_2000, Assaad_2022} aim to infer these relations, yet traditional approaches depend on untestable assumptions, need large data sets, and are unsatisfactory for real-world IT monitoring~\citep{Ait_Bachir_2023}. This is especially true when causal relations are event-driven rather than continuous.
	%
	
	% With advancements in storage and cloud technology, substantial historical data becomes available, aiding in understanding how anomalies propagate and expediting root cause analysis (RCA). 
	% Various techniques for RCA in IT systems use causal discovery methods~\citep{Spirtes_2000}, which infer causal graphs from anomalous time series. 
	% However, these methods are often too slow for swift root cause identification. Additionally, classical time series causal discovery methods~\citep{Assaad_2022} prove unsatisfactory for real-world IT monitoring~\citep{Ait_Bachir_2023}.
	% Another approach assumes access to the true causal graph representing the normal system regime, using this graph and data to infer root causes. However, obtaining the causal graph from experts can be impractical, leading to the need to infer the graph from normal time series, which may not always be satisfactory for real-world IT monitoring. This is particularly true when continuous causal relations do not exist in the system, and instead, they manifest in an event-driven manner.
	%
	%Binary thresholding transforms continuous time series into binary sequences, emphasizing threshold-crossing events like anomalies. 
	Understanding causal relations between binary time series offers clearer insights than using raw time series alone. In this paper, we transform raw time series into binary series using thresholds, then discover a causal graph from these binary series, showing causal relationships between threshold crossings. Root causes are then detected using graph traversal techniques.
	Our contributions are summarized as follows:
	\begin{enumerate}
		\item We propose a structural causal model for anomaly propagation in threshold-based IT monitoring systems, incorporating two distinct noise terms to simulate signal dissipation and external interventions, and providing insight into the event-based nature of anomaly propagation.
		\item We introduce the T-RCA framework for detecting root causes, capable of identifying all true root causes under certain assumptions, and provide an empirical solution for scenarios where one assumption is not met.
		\item We conduct extensive experiments with synthetic and real data to evaluate the method.
	\end{enumerate}
	The remainder of this paper is organized as follows: Section~\ref{sec:sota} discusses related work. Section~\ref{sec:TDSCM} introduces a new structural causal model for threshold-based IT systems. Section~\ref{sec:TRCA} presents an algorithm to detect root causes of anomalies in these systems. In Section~\ref{sec:experiments}, the method is compared to others on simulated and real datasets. %Section~\ref{sec:deployment} introduces deployment details. 
	Finally, Section~\ref{sec:discussion} concludes the paper.
	
	\section{State of the Art}\label{sec:sota}

	\begin{figure*}[t!]
		\begin{tabular}[c]{cc}
			\multirow{2}{*}[4.cm]{
				\begin{subfigure}{.5\textwidth} % .5
					\centering
					\includegraphics[trim = 0 1cm 0 1cm, clip=TRUE]{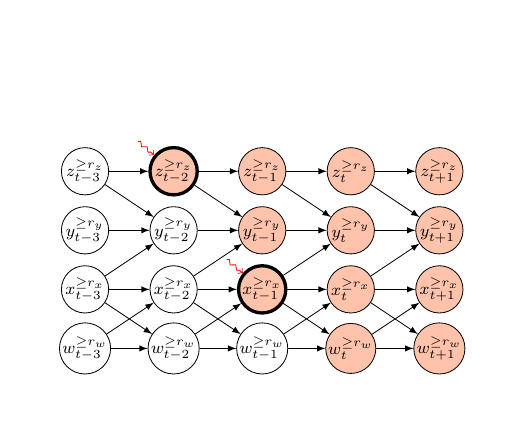} % scale=1.0,
					\caption{Finite Threshold-based Full-time Causal Graph, T-FTCG. Orange vertices are anomalies,  thick border vertices are root causes, and the variable $i$ is  displayed by squiggly red arrows when equal to 1.}
					\label{fig:window-graph}
				\end{subfigure}
			}& 
			\begin{subfigure}{.45\textwidth}
				\centering
				%   \begin{tikzpicture}[{black, circle, draw, inner sep=0}]
					% \tikzset{nodes={draw,rounded corners},minimum height=0.8cm,minimum width=0.8cm, font=\footnotesize}
					% \tikzset{latent/.append style={fill=gray!30}}
					
					% \node[fill=easyorange!40, ultra thick] (Z) at (-1.5,0) {$Z^{\ge r_z}$} ;
					% \node[fill=easyorange!40] (Y) at (0,0) {$Y^{\ge r_y}$};
					% \node[fill=easyorange!40, ultra thick] (X) at (1.5,0) {$X^{\ge r_x}$};
					% \node[fill=easyorange!40] (W) at (3,0) {$W^{\ge r_w}$};
					
					% \draw[->,>=latex] (Z) -- (Y);
					% \draw[->,>=latex] (X) -- (Y);
					% \draw[->,>=latex] (W) to [bend left=10] (X);
					% \draw[->,>=latex] (X) to [bend left=10] (W);
					
					% \draw[->,>=latex] (Y) to [out=170,in=130, looseness=2] (Y);
					% \draw[->,>=latex] (Z) to [out=170,in=130, looseness=2] (Z);
					% \draw[->,>=latex] (X) to [out=170,in=130, looseness=2] (X);
					% \draw[->,>=latex] (W) to [out=170,in=130, looseness=2] (W);
					
					% \end{tikzpicture}
				\includegraphics[trim = 0 1cm 0 1cm, clip=TRUE]{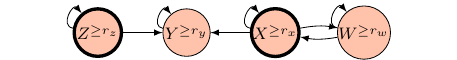}
				\caption{Threshold-based Summary Causal Graph, T-SCG. Orange vertices are anomalies and thick border vertices are root causes.}
				\label{fig:summary-graph}
			\end{subfigure}\\
			&
			\begin{subfigure}{.41\textwidth}
				\centering 
				\includegraphics[scale=0.8, trim = 0 1cm 0 1cm, clip=TRUE]{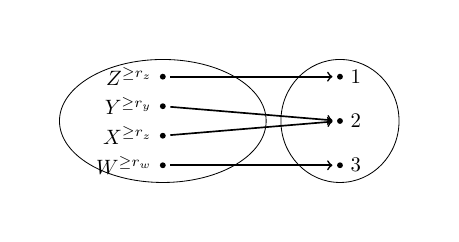}
				\caption{Mapping for the appearance time of anomalies.}
				\label{fig:appearance}
			\end{subfigure}
		\end{tabular}
		\caption{Example. Illustration of (a) a T-FTCG, (b) a T-SCG and (c) the mapping for the appearance time of anomalies on a system with four variables. }
		\label{fig:diamond}
	\end{figure*}%
	
	Considerable research efforts have recently focused on automated root cause analysis. Some methods utilize causal discovery to uncover causal graphs from anomalous data. For instance, CloudRanger \citep{wang2018cloudranger} employs the PC algorithm \citep{Spirtes_2000} to discover causal graphs among anomalous time series, then identifies root causes using a random walk strategy based on transition matrices calculated from time series correlations. However, the PC algorithm was not originally designed for time series data, and correlations may not fully capture authentic causal effects. MicroCause \citep{Meng_2020} addresses these concerns by utilizing the PCMCI algorithm \citep{Runge_2019} for discovering the causal graph and partial correlations instead of correlations. %While these methods necessitate inferring a causal graph whenever an anomaly arises, causal discovery methods rely on untestable assumptions, require substantial data, and demonstrate suboptimal efficiency, particularly with quantitative time series data from IT systems \citep{Ait_Bachir_2023}.
	
	Other approaches focus on detecting root causes given the true causal graph of the normal regime. For example, EasyRCA \citep{Assaad_2023} identifies some root causes directly from the graph and detects the rest by comparing direct effects in normal and anomalous regimes. CIRCA \citep{Li_2022} employs a service through a graph structure and conducts a regression hypothesis test on anomalous data to identify deviations. However, these methods require a graph as input, which may not always be available.
	
	In event-driven systems, event-based causal relations may be more useful. \cite{van2021root} introduced a method, called AITIA-PM, for identifying root causes within event logs \citep{rudnitckaia2016process}. This approach leverages a test hypothesis that considers root causes as variables with the highest conditional dependence with anomalies. However, this method is not sound theoretically. Similarly, RCD \citep{RCD} focuses on discretizing data and identifying root causes through the PC algorithm. However, similarly to CloudRanger and MicroCause, it needs to run the causal discovery algorithm each time an anomaly arises.
	
	From a broader perspective,  \cite{Wang2023reason} has proposed to model complex systems with interdependent network structures, and detect root causes using hierarchical graph neural networks. 
	\cite{HRLHF} has used some human feedback in a reinforcement learning fashion to reduce the number of queries. 
	Some methods have also been proposed for inferring root causes of anomalies for non-temporal data, including \cite{Budhathoki_2021,Budhathoki_2022}. Those studies are beyond the scope of this paper.

	%%%%%%%%%%%%%%%%%%%%%%%%%%%%%%%%%%%%%%%%%%%%%%%%%%%%%%%%%%

	\section{Threshold-based causal graphs and root causes}
	\label{sec:TDSCM}

	In this section, we present key concepts and assumptions. Lowercase letters represent observed variables, uppercase letters denote name-values or time series, blackboard bold letters indicate sets, and Greek letters represent constants. We denote  $\mathbb{1}_A$ as the indicator function of the event $A$, ${\mathcal{G}}$ as a graph, $\mathcal{B}$ the Bernoulli distribution, $\mathrm{Pa}_{\mathcal{G}}(X)$, $\mathrm{An}_{\mathcal{G}}(X)$ and $\mathrm{Desc}_{\mathcal{G}}(X)$ as the sets of parents, ancestors and descendants of a vertex $X$ in ${\mathcal{G}}$, respectively. 
	
	In IT systems, the data associated to diverse components of the system is commonly gathered in the form of time series. %, formally defined as follows.
	\begin{definition}[Time series]
		For $t\in \mathbb{N}$, consider the random variable $x_t \in [0, 1]$. The sequence $\mathcal{X}=\{x_t; t \in \mathbb{N}\}$ is called a \emph{discrete time series}. 
		Let $\mathbb{V}$ be the set of name-values of $d$ different discrete time series in a system,  $\mathbb{T}=\{\mathcal{X}=\{x_t; t \in \mathbb{N}\}; X \in \mathbb{V}\}$ is called a \emph{$d$-dimensional discrete time series}.
	\end{definition}
	
	In practical scenarios, time series within IT systems are not observed across an infinite set of time points due to operational constraints, such as specific timeframes, resource limits, data storage, and monitoring capabilities with finite recording capacities or designated data collection periods. Subsequently, we consider discrete-time time series with continuous values unless explicitly stated otherwise.
	
	In many IT systems, establishing a  causal connection between two time series $\mathcal{X}$ and $\mathcal{Y}$ is not evident at every time step. An anomaly in $\mathcal{X}$ can result in an anomaly in $\mathcal{Y}$ when a specific time point $x_t \in\mathcal{X}$ exceeds a predefined threshold, triggering a corresponding time point $y_t \in \mathcal{Y}$ to breach its own threshold. For example, changes in network traffic may not impact firewall alerts at each time point, but when network traffic surpasses a threshold, it could indicate a security threat, increasing firewall alerts to signal potential intrusion attempts. In such cases, relying solely on raw time series is challenging, requiring binary thresholding for time series.
	
	\begin{definition}[Binary thresholding of time series]
		Consider a discrete time series $\mathcal{X}=\{x_t; t \in \mathbb{N}\}$ and a fixed threshold $r_x\in [0, 1]$. A \emph{binary thresholding of $\mathcal{X}$} is the sequence $\mathcal{X}^{\ge r_x}=\{\mathbb{1}_{x_t\ge r_x}; t \in \mathbb{N}\}$.
	\end{definition}
	Each binary random variable in the sequence $\mathcal{X}^{\ge r_x}$ is represented as $x_t^{\ge r_x}$, where $t \in \mathbb{N}$, and the binary thresholding of a $d$-dimensional time series $\mathbb{T}$ is the vector comprising the binary thresholdings of the respective time series.
	%
	%Investigating causal connections among binary thresholded time series provides a more effective and intuitive perspective for understanding causal relations in these systems.
	The subsequent definition introduces the concept of a causal graph over binary thresholded $d$-dimensional time series, illustrated in Figure \ref{fig:window-graph}.
	
	\begin{definition}[Threshold-based full-time causal graph, T-FTCG]
		\label{def:t-FTCG}
		Let $\mathbb{T}$ be a $d$-dimensional time series in a system, $\mathbb{V}$ the set of their name-values, and $r\in \mathbb{R}^d$ a vector of thresholds. A \emph{threshold-based full time causal graph} $\mathcal{G}_{\text{ft}}=(\mathbb{T}^r, \mathbb{E}_{\mathbb{T}}^r)$ is an infinite directed acyclic graph where the set of vertices $\mathbb{T}^r$ corresponds to the set of variables in the binary thresholding of $\mathbb{T}$ and where the set of edges $\mathbb{E}_{\mathbb{T}}^r$ is defined as follows:  for two time series $\mathcal{X},\mathcal{Y} \in \mathbb{T}$, $\forall x_{t-\gamma} \in \mathcal{X}$, and $\forall y_t \in \mathcal{Y}$,  $x_{t-\gamma}^{\ge r_x} \rightarrow y_{t}^{\ge r_y}$ in $\mathbb{E}_{\mathbb{T}}^r$ if and only if $x_{t-\gamma}^{\ge r_x}$ causes $y_t^{\ge r_y}$ at time $t$ with a time lag of $\gamma >0$ (no instantaneous causal relations).
	\end{definition}
	
	%An example of a T-FTCG is depicted in Figure \ref{fig:window-graph}. % It is important to note that, as per our T-FTCG definition, causal relations cannot be instantaneous. %However, in real-world applications, instances of apparent instantaneous causal relations may arise; for instance, a low sampling rate might create the illusion of instantaneous causal connections due to the limitations of temporal resolution.
	%In Section~\ref{sec:discussion}, we delve into why %this consideration holds significance in our work and explore how it can be relaxed. \textcolor{red}{Section 6??}
	%
	To establish a connection between the T-FTCG and the observational data, we adopt the following standard assumptions. 
	
	\begin{assumption}[Causal Markov condition]
		\label{assumption:cmc}
		Let $\mathcal{G}_{\text{ft}}=(\mathbb{T}^r,\mathbb{E}_{\mathbb{T}}^r)$ be a T-FTCG. For each  thresholded time series $\mathcal{X}^{\geq r_x} \in \mathbb{T}^r$, each vertex $x_t^{\ge r_x}$ is independent of its non descendants in $\mathcal{G}_{\text{ft}}$ given its parents.
	\end{assumption}

	\begin{assumption}[Adjacency faithfulness]
		\label{assumption:adjacency_faithfulness}
		Let $\mathcal{G}_{\text{ft}}=(\mathbb{T}^r,\mathbb{E}_{\mathbb{T}}^r)$ be a T-FTCG. Each two adjacent vertices are statistically dependent given any set of vertices.
	\end{assumption}

	While a T-FTCG only illustrates the causes of an effect, the threshold-based dynamic structural causal model (an adaptation of the structural causal models introduced in \cite{Pearl_2000} to threshold-based systems) describes how an effect is quantitatively influenced by its causes.

	\begin{definition}[Threshold-based dynamic structural causal model, T-DSCM]
		\label{def:tdscm}
		A \emph{threshold-based dynamic structural causal model}  associated with a T-FTCG $\mathcal{G}_{\text{ft}}=(\mathbb{T}^r, \mathbb{E}_{\mathbb{T}}^r)$  is a quadruple $\mathcal{M}=\langle (\mathbb{U}_t, \mathbb{I}_t), \mathbb{T}^r, {f}, (P(u_t), P(i_t)) \rangle$ where 
		\begin{enumerate}
			\item  $\mathbb{U}_t$ and $\mathbb{I}_t$ are two sets of $d$-dimensional binary background time series (also called exogenous time series), such that  $u_t^y\in U^y_t \in \mathbb{U}_t$ and $i^y_t\in I_t^y \in \mathbb{I}_t$ are determined by factors outside the model, for $Y\in \mathbb{V}$ and $t\in \mathbb{N}$;
			\item $\mathbb{T}^r$ is a $d$-dimensional binary observed time series (also called endogenous time series), such that each binary variable $y_t^{\ge r_y}\in \mathcal{Y}^{\ge r_y}\in \mathbb{T}^r$ is determined by variables in the model, for $t\in \mathbb{N}$; 
			\item for $\mathcal{Y}^{\ge r_y} \in \mathbb{T}^r$ and $t \in \mathbb{N}$, the structural mapping function is given by, for $u_t^y\in U^y_t\in \mathbb{U}_t$ and $i^y_t\in I_t^y\in \mathbb{I}_t$,
			\begin{align*}
				y_t^{\ge r_y} &:= f(Pa_{\mathcal{G}_{\text{ft}}}(y_t^{\ge r_y}),u^{y}_t ,i^{y}_t)\\
				&:= \left(\left(\bigvee_{x_{t-\gamma}^{\ge r_x}\in Pa_{\mathcal{G}_{\text{ft}}}(y_t^{\ge r_y})} x_{t-\gamma}^{\ge r_x}\right) \wedge u^{y}_t\right) \vee i^{y}_t;
			\end{align*}
			% where $PA^y\subseteq \mathbb{V}$ are the direct causes of $\mathcal{Y}$ within the system;
			\item for all $u^{y}_t \in U^y_t \in \mathbb{U}_t$, $P(u^{y}_t)\sim\mathcal{B}(\epsilon^y_t)$ with $\epsilon^y_t \in (0,1]$ and for all $i^{y}_t \in I_t^y \in \mathbb{I}_t$,  $P(i^{y}_t)\sim\mathcal{B}(\beta^y_t)$ with $\beta^y_t \in (0,1)$.
		\end{enumerate}
	\end{definition}
	
	\begin{figure*}
		\centering
		\includegraphics[width=0.8\textwidth, trim = 1.6cm 3.05cm 1.1cm 2.8cm, clip=TRUE]{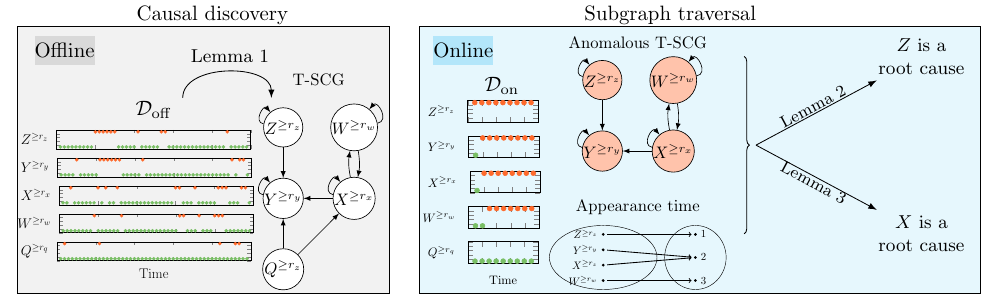}
		\caption{Overview of T-RCA. First step, on the offline dataset: a T-SCG is learned from $\mathcal{D}_{\text{off}}$. Second step, the anomalous T-SCG is deduced from the online dataset, as well as the appearance time. Last step, detection of the root causes using Lemmas \ref{lemma:root_causes_forwards} and \ref{lemma:root_causes_SCC}.}
		\label{fig:architecture}
	\end{figure*}
	
	Note that the background variable $u^y_t\sim \mathcal{B}(\epsilon^y_t)$ denotes the non-immunity of $y_{t}^{\ge r_y}$: if $u^y_t=0$, then $y_{t}^{\ge r_y}$ is unaffected by its parents in the T-FTSCG. When $\epsilon^y_t=1$, an anomaly will deterministically result in another anomaly. We avoid $\epsilon^y_t=0$ to maintain the causal relation between $y^{\ge r_y}$ and its parents.
	Similarly, the background variable $i^y_t\sim \mathcal{B}(\beta^y_t)$ represents a hidden cause. If $i^y_t=1$, then $y_{t}^{\ge r_y}=1$ independently of its parents in the T-FTSCG. We avoid $\beta^y_t=0$ (or $\beta^y_t=1$) to prevent $y^{\ge r_y}=1$ for all $t$ (or $0$).
	
	The T-DSCM relies on the following standard assumption, implying the absence of  \emph{unobserved} confounding biases in the system.

	\begin{assumption} 
		\label{assumption:no_hidden_con}
		Let $\mathcal{M}=\langle(\mathbb{U}_t, \mathbb{I}_t), \mathbb{T}^r, f, (P(u_t), P(i_t))\rangle$ be a T-DSCM.
		We assume that all variables $u_t^y$ and $i_t^y$ of the background time series $\mathbb{U}_t$ and $\mathbb{I}_t$ in $\mathcal{M}$ are jointly independent.
	\end{assumption} 
	
	Now we can define anomalies and root causes when considering a T-DSCM associated to a T-FTCG.
	
	\begin{definition}[Anomaly and root cause]
		\label{def:anomaly_TDSCM_version}
		Let $\mathcal{G}_{\text{ft}} = (\mathbb{T}^r, \mathbb{E}_{\mathbb{T}}^r)$ a T-FTCG and $\mathcal{M}= \langle(\mathbb{U}_t, \mathbb{I}_t), \mathbb{T}^r, f, (P(u_t), P(i_t))\rangle$ its T-DSCM, and a sample satisfying $\mathcal{M}$.
		The random variable $x_t^{\ge r_x} \in \mathbb{T}^r$ is said to be \emph{anomalous} if $x^{\ge r_x}_t=1$ and a \emph{root cause} if $i^x_t=1$.
	\end{definition} 
	
	Graphically, we represent anomalies by orange nodes and root causes by thick borders, as depicted in  Figure~\ref{fig:diamond}.
	By definition of a T-DSCM, each variable may be a root cause ($\epsilon^y_t\in (0,1)$), %For each $i_t^y \in I_t^y$, where $I_t^y\in \mathbb{I}_t$, $P(i^y_t)=$ Bernoulli($\epsilon^y_t$) with $\epsilon^y_t\in (0,1)$. 
	all root causes are anomalies (if $i^y_t=1$ then $y^{\ge r_y}_t=1$), and
	each anomaly is either propagated from a root cause through a chain of anomalies or is itself a root cause: if $y^{\ge r_y}_t=1$ either $i^y_t=1$ or $\exists x_{t-\gamma}^{\ge r_x}\in Pa_{\mathcal{G}_{ft}}(y^{\ge r_y}_t)$ such that  $x_{t-\gamma}^{\ge r_x} =1$ and $u_t^y=1$. 
	%Each $u_t^y \in U_t^y$, where $U_t^y \in \mathbb{U}_t$, determines if $y_t$ is immune to the propagation of anomalies through its parents or not, with $\epsilon_t^y$ defining the probability of non-immunity.
	
	Our main goal is to directly identify root causes from time series data using predefined thresholds. We aim to utilize causal discovery algorithms, leveraging historical data to infer a causal graph, but the T-FTCG is infinite. Under the following condition, it becomes finite, sometimes referred to as a window causal graph \citep{Assaad_2022}.
	
	\begin{assumption}
		\label{assumption:consistency}
		Let $\mathcal{G}_{\text{ft}}$ be a T-FTCG. All the causal relationships remain constant in direction throughout time in $\mathcal{G}_{\text{ft}}$. 
	\end{assumption}
	
	This assumption implies that edges in the T-FTCG remain consistent over time (consistency throughout time). We denote $\gamma_{\max}$ as the maximal temporal lag between a cause and its effect within the system.
	Understanding and validating the finite T-FTCG can be challenging for system experts in many applications, and errors may arise for finite sample size and large $\gamma_{\max}$,  in estimation or when determining the exact lag between a cause and its effect. We introduce  an abstraction of the T-FTCG illustrated in Figure~\ref{fig:summary-graph}. 
	
	\begin{definition}[Threshold-based Summary Causal Graph, T-SCG]
		\label{def:t-SCG}
		Let $\mathbb{V}$ be the set of name-value of $d$ different time series in a system, and  $\mathcal{G}_{\text{ft}}=(\mathbb{T}^r,\mathbb{E}_{\mathbb{T}}^r)$ the corresponding T-FTCG. The \emph{threshold-based summary causal graph} $\mathcal{G} = (\mathbb{V}^r, \mathbb{E}^r)$ associated to $\mathcal{G}_{\text{ft}}$ is given by $\mathbb{V}^r := \{X^{\ge r_x} \mid \forall X\in \mathbb{V} \text{ and }  r_x\in \mathbb{R}^{d}\}$ and $\mathbb{E}^r$ such that $X^{\ge r_x} \rightarrow Y^{\ge r_y}$ is in $\mathbb{E}^r$ if and only if there exists $\gamma\in \{1, \cdots, \gamma_{\max}\}$ such that $x_{t-\gamma}^{\ge r_x} \rightarrow y_t^{\ge r_y}$ in $\mathbb{E}_{\mathbb{T}}^r$.
	\end{definition}
	
	%For illustration, we present in Figure~\ref{fig:summary-graph} the T-SCG associated to the T-FTCG introduced in Figure~\ref{fig:window-graph}.
	%While the T-FTCG is presumed to be acyclic, its corresponding T-SCG can be  cyclic.

	Finally, we introduce the appearance time of anomalies, illustrated in Figure \ref{fig:appearance}.

	\begin{definition}[Appearance time of anomalies]
		\label{def:appearance_time}
		Given a T-SCG $\mathcal{G}=(\mathbb{V}^r, \mathbb{E}^r)$ and observations of the time series, the appearance time of anomalies is the mapping 
		$  \tau: 
		X^{\ge r_x} \mapsto \underset{t}{\operatorname{argmin}} \{ x_t^{\ge r_x}=1\}.$
	\end{definition}
	
	%This notion of appearance time of anomalies is %
	%

	\section{Threshold-based root cause analysis}
	\label{sec:TRCA}
	
	In IT monitoring systems, our primary focus is typically on identifying the root causes of existing anomalies, specifically those presently occurring. Thus, we differentiate between historical data, denoted as offline data $\mathcal{D}_\text{off}$, and current data, referred to as online data $\mathcal{D}_\text{on}$. Anomalies and root causes will be sought within the online data $\mathcal{D}_\text{on}$.
	
	We propose T-RCA (Threshold-based Root Cause Analysis) for detecting root causes. This method involves causal discovery to uncover the T-SCG from offline data $\mathcal{D}_{\text{off}}$. When anomalies occur in online data $\mathcal{D}_{\text{on}}$, a graph traversal strategy is applied to the inferred T-SCG. An overview is provided in Figure~\ref{fig:architecture}. Section~\ref{ssec:cd} outlines causal discovery, while Section~\ref{ssec:TRCA} details graph traversal and presents the result that, under an additional assumption, our method identifies root causes precisely. Section~\ref{ssec:TRCAextentions} introduces an extension applicable when this assumption is violated.
	
	\subsection{Causal discovery of T-SCG}
	\label{ssec:cd}

	According to \cite{Pearl_1988}, under certain assumptions and in the absence of instantaneous relations, it's possible to discover a directed and acyclic causal graph from observational data. We restate this result precisely within the framework of T-FTCG. %, with the proof following from \cite{Pearl_1988}.
	
	%\begin{restatable}{lemma}{mylemmaone}
	\begin{lemma}
		\label{lemma:identifiability_of_TFTCG}
		Let $\mathcal{M}$ be a T-DSCM associated to a T-FTCG  $\mathcal{G}_{\text{ft}}$. 
		If Assumptions~\ref{assumption:cmc},~\ref{assumption:adjacency_faithfulness},~\ref{assumption:no_hidden_con},~\ref{assumption:consistency} are satisfied then
		$\mathcal{G}_{\text{ft}}$ is identifiable from the distribution induced by $\mathcal{M}$.
	\end{lemma}
	%\end{restatable} 
	
	%The proof of this lemma follows directly from \cite{Pearl_1988}.
	
	% \begin{sproof}
		% We present here a sketch of proof, and refer the interested reader to  \cite{Pearl_1988} for more details.
		%     The identifiability of the causal graph is given by the temporal order of variables, which states that a cause precedes the effect in time. Additionally, according to the causal Markov condition, we can effectively determine the independence between a vertex and its non-descendants given its parents. Consider the historical dataset $\mathcal{D}_{\text{off}}$. We can initialize  a graph $\mathcal{\hat G}_{\text{ft}}$ where, for all vertices $(y_t^{\ge r_y}, x_{t-\gamma}^{\ge r_x}) \in \mathbb{T}^r$, for $\gamma >0$, $x_{t-\gamma}^{\ge r_x} \rightarrow y_t^{\ge r_y}$. Then, we test conditional independences to remove edges, with the adjacency faithfulness assumption ensuring that the excluded variables are not parents of $y_t^{\ge r_y}$. This process leaves only those variables that persist as parents of $y_t^{\ge r_y}$ in $\mathcal{\hat G}_{\text{ft}}$. 
		% \end{sproof}
	
	In practice, any causal discovery algorithm capable of handling binary time series and aligned with our assumptions can be employed. 
	Once we get the inferred T-FTCG $\mathcal{\hat G}_{\text{ft}}$, we can easily deduce the T-SCG $\mathcal{\hat G}$ using Definition~\ref{def:t-SCG}.
	It is also possible to directly discover the T-SCG (assuming no instantaneous relations) \citep{Assaad_2022PCGCE}\footnote{The idea proposed in \cite{Assaad_2022PCGCE} consider a richer type of graph called extended summary causal graph, but when there are no instantaneous relations, the summary causal graph gives the same information as the extended summary causal graph.}.
	
	\subsection{Root cause detection via subgraph traversal}
	\label{ssec:TRCA}
	
	Assuming we have a T-SCG $\mathcal{G}=(\mathbb{V}^r, \mathbb{E}^r)$ and online data $\mathcal{D}_{\text{on}}$ containing observed anomalies, our objective is root cause detection. We can construct a colored graph, where vertices are colored if anomalies are present in $\mathcal{D}_{\text{on}}$. These colored vertices are denoted as $\mathbb{A}$, and $\mathcal{G}_{\mathbb{A}}$ represents the anomalous subgraph containing only these colored vertices. We first proceed by decomposing this graph into strongly connected components.
	
	\begin{definition}[Strongly connected component,  (SCC)]
		\label{def:SCC}
		Let $\mathcal{G}=(\mathbb{V}^r, \mathbb{E}^r)$ be a T-SCG. A subset $\mathbb{S} \subseteq \mathbb{V}^r$ is a \emph{strongly connected component} of $\mathcal{G}$ iff $\mathbb{S}$ is a maximal set of vertices where every vertex is reachable via a directed path in $\mathcal{G}$  from every other vertex in $\mathbb{S}$.
	\end{definition}
	\noindent 
	For instance, in Figure~\ref{fig:summary-graph}, there are three SCCs: $\{Z^{\ge r_z}\}$, $\{Y^{\ge r_y}\}$, and $\{X^{\ge r_y}, W^{\ge r_w}\}$.
	%
	%When there is only one root cause per strongly connected component, it becomes possible to detect it using the following lemma.
	For each SCC of size $1$, we use the following lemma to test if the vertex in the SCC is a root cause.
	
	%The first case of root cases is presented in the following lemma,  directly deduced from Definition~\ref{def:t-FTCG}.
	
	%\begin{restatable}{lemma}{mylemmatwo}
	\begin{lemma}
		\label{lemma:root_causes_forwards}
		{Let $\mathcal{G}_{\mathbb{A}}$ be an anomalous subgraph of a T-SCG $\mathcal{G}=(\mathbb{V}^r, \mathbb{E}^r)$ in $\mathcal{D}_{on}$}.
		If an anomalous vertex in $\mathcal{G}_{\mathbb{A}}$ does not have any anomalous parent in $\mathcal{
			G}_{\mathbb{A}}$ then it is a root cause.
	\end{lemma}
	%\end{restatable} 
	
	\begin{proof}
		By Definition~\ref{def:tdscm}, the anomaly on a vertex $Y^{\ge r_y}$ that does not have any anomalous parent cannot be propagated from other vertices. Thus $i^{y}_t=1$ which implies that $Y^{\ge r_y}$ is a root cause.
	\end{proof}
	
	This lemma offers a direct method to identify certain root causes. For instance, in Figure \ref{fig:architecture}, it detects $Z^{\ge{r_z}}$ as a root cause. 
	
	Furthermore, using the notion of appearance time of anomalies, for each SCC of size greater than $1$, we use the following lemma to test if one of the vertices in the SCC is a root cause. 
	
	% Given that we have access to the appearance time of anomalies, we can still manage to detect all root causes in this category. \textcolor{green}{ajouter: dans un SCC, lemme 2 n'est pas applicable} 
	
	%\begin{restatable}{lemma}{mylemmathree}
	\begin{lemma}
		\label{lemma:root_causes_SCC} {Let $\mathcal{G}_{\mathbb{A}}$ be an anomalous subgraph of a T-SCG $\mathcal{G}=(\mathbb{V}^r, \mathbb{E}^r)$ in $\mathcal{D}_{on}$}, $\tau$ be the appearance time of anomalies, and $\mathbb{S}$ an SCC in $\mathcal{G}_{\mathbb{A}}$ of $size(\mathbb{S})>1$ such that  {$Pa_{\mathcal{G}_{\mathbb{A}}}(\mathbb{S})\subseteq \mathbb{S}$}. The vertex
		$ \underset{X^{\ge r_x}\in \mathbb{S}}{\operatorname{argmin}}  \{\tau(X^{\ge r_x})\}$
		is a root cause. 
	\end{lemma}
	%\end{restatable}
	
	\begin{proof}
		Let $\mathbb{S}$ be an SCC in an anomalous subgraph $\mathcal{G}_{\mathbb{A}}$ such that $size(\mathbb{S})>1$ and $Pa_{\mathcal{G}_{\mathbb{A}}}(\mathbb{S})\subseteq \mathbb{S}$. By  Definition~\ref{def:tdscm}, the anomaly on $Y^{\ge r_y}$ satisfying 
		$\underset{X^{\ge r_x}\in \mathbb{S}}{\operatorname{argmin}}  \{\tau(X^{\ge r_x})\}$
		cannot be propagated from other vertices, i.e., 
		$\left(\left(\bigvee_{x_{t-\gamma}^{\ge r_x}\in Pa_{\mathcal{G}_{\text{ft}}}(y_t^{\ge r_y})} x_{t-\gamma}^{\ge r_x}\right) \wedge u^{y}_t\right)=0.$ Thus $i^{y}_t=1$ which implies that $Y^{\ge r_y}$ is a root cause.
	\end{proof}
	
	In the example in Figure \ref{fig:diamond},  we can deduce from this lemma that $X^{\ge{r_x}}$ is a root cause since $\tau(X^{\ge{r_x}})<\min(\tau(Y^{\ge{r_y}}),\tau(W^{\ge{r_w}}))$.
	
	%This assumption is needed for our main algorithm T-RCA. Nevertheless, in Section~\ref{ssec:TRCAextentions}, we introduce an extension of T-RCA that can be applied when Assumption~\ref{assumption:one_intervention} is not satisfied.

	% Under this assumption we get this lemma which is converse of Lemma~\ref{lemma:root_causes_forwards}.
	% \begin{restatable}{lemma}{mylemmathree}
		% \label{lemma:root_causes_backwards}
		% Let $\mathcal{G}=(\mathbb{V}^r, \mathbb{E}^r)$ be a T-SCG.
		% If a vertex $X^{\ge r_x}$ is a root cause and there exists no SCC $\mathbb{S}$ such that $X^{\ge r_x}\in \mathbb{S}$ and $size(\mathbb{S})=1$ then $X^{\ge r_x}$ is an anomalous vertex in the T-SCG that do not have any anomalous parent in $\mathcal{G}$.
		% \end{restatable}

	It is noteworthy that the aforementioned lemmas generally do not guarantee the detection of all root causes. Specifically, these lemmas are incapable of detecting a root cause that is influenced by another root cause, i.e., $i_t^y=1$ and simultaneously $\exists x_{t-\gamma}^{\ge r_x}\in Pa_{\mathcal{G}_{\text{ft}}}(y_t^{\ge r_y})$ such that $ x_{t-\gamma}^{\ge r_y} \wedge u^{y}_t=1.$
	This means that in Figure~\ref{fig:window-graph}, if $z^{\ge r_z}_{t-2}$ is a root cause then $y^{\ge r_y}_{t-1}$ cannot be a root cause.
	We argue that these undetectable root causes are rare in practice, therefore we introduce the following assumption to mitigate their impact.
	% \footnote{
		% It might appear that under Assumption~\ref{assumption:one_intervention}, the time of anomaly and/or the association (possibly conditional) between anomalies might be sufficient to detect root causes. However, it turns out that this is not true as in it show in Appendix \ref{app:ex} via several examples, highlighting the crucial role of the causal discovery step.
		% }

	% \begin{assumption}
		% \label{assumption:one_intervention}
		% Let $\mathcal{M}=\langle(\mathbb{U}_t, \mathbb{I}_t), \mathbb{T}^r, f, (P(u_t), P(i_t))\rangle$ be a T-DSCM associated to a T-FTCG $\mathcal{G}_{\text{ft}}$. 
		% We assume that if $x^{\ge r_x}_t$ is a root cause, i.e., $i^x_t=1$,  then there exists no $z^{\ge r_z}_{t'}\in An_{\mathcal{G}_{\text{ft}}}(x^{\ge r_x}_t)$, for $t'<t$, such that 
		% \begin{itemize}
			% \item $z^{\ge r_z}_{t'} $ is a root cause, i.e. $i^{\ge z}_{t'}=1$; and
			% \item  for all $y^{\ge r_y}_{t'}\in Desc_{\mathcal{G}_{\text{ft}}}(z^{\ge r_z}_{t'})\cap An_{\mathcal{G}_{\text{ft}}}(x^{\ge r_x}_t)\backslash\{x^{\ge r_x}_t,z^{\ge r_z}_{t'}\}$, $y^{\ge r_y}_{t'}=1$.
			% \end{itemize}
		% \end{assumption}
	\begin{assumption}
		\label{assumption:one_intervention}
		Let $\mathcal{G}=(\mathbb{V}^r, \mathbb{E}^r)$ be a T-SCG. 
		We assume that if $X^{\ge r_x}$ is a root cause in $\mathcal{D}_{on}$, i.e., $\exists t\in \mathcal{D}_{on}$ such that $i^{x}_{t}=1$,  then there exists no $Z^{\ge r_z}\in An_{\mathcal{G}}(X^{\ge r_x})$, such that 
		\begin{itemize}
			\item $Z^{\ge r_z}$ is a root cause in $\mathcal{D}_{on}$; and
			\item  for all $Y^{\ge r_y}\in Desc_{\mathcal{G}}(Z^{\ge r_z})\cap An_{\mathcal{G}}(X^{\ge r_x})\backslash\{X^{\ge r_x},Z^{\ge r_z}\}$, $Y^{\ge r_y}$ is anomalous in $\mathcal{D}_{on}$.
		\end{itemize}
	\end{assumption}
	
	Generally, a T-FTCG offers a more detailed system representation than a T-SCG. However, due to the absence of instantaneous relations and the timing of anomalies, along with Assumption~\ref{assumption:one_intervention}, the T-FTCG doesn't provide an advantage in root cause detection over the T-SCG. The following theorem verifies the validity of T-RCA\footnote{The pseudo-code of T-RCA is available in Appendix.} under Assumption~\ref{assumption:one_intervention}.
	%
	%\begin{restatable}{theorem}{mytheoremone}
	\begin{theorem}
		\label{theorem:TRCA}
		Under Assumptions~\ref{assumption:cmc},~\ref{assumption:adjacency_faithfulness},~\ref{assumption:no_hidden_con},~\ref{assumption:consistency},~\ref{assumption:one_intervention}, T-RCA detects exactly the set of true root causes: if $\mathbb{C}$ is the set of true root causes and $\mathbb{\hat C}$ is the set of root causes inferred by T-RCA, $\mathbb{C}=\mathbb{\hat C}$.
	\end{theorem}
	%\end{restatable}
	%
	\begin{proof}
		The soundness' proof (an inferred root cause by T-RCA is  a true root cause)  is directly given  by Lemmas~ \ref{lemma:root_causes_forwards} and \ref{lemma:root_causes_SCC}.
		
		Let prove  that under~Assumption~\ref{assumption:one_intervention}, a true root cause is necessarily inferred by T-RCA.
		Suppose that there exists $C\in \mathbb{C}$ such that $C$ is not detectable by T-RCA. 
		$C$ is not a root vertex in the anomalous subgraph of the T-SCG, otherwise it would have been detected by Lemma~\ref{lemma:root_causes_forwards}. $C$ does not belongs to an SCC $\mathbb{S}$ where $\mathbb{S}$ has a size greater than $1$ and $Pa_{\mathcal{G}_{\mathbb{A}}}(\mathbb{S})\subseteq \mathbb{S}$, otherwise it would have been detected by Lemma~\ref{lemma:root_causes_SCC}. 
		Now, if $C$ is not a root vertex and  belongs to an SCC $\mathbb{S}$ of size $1$ it violates Assumption~\ref{assumption:one_intervention}, since in this case, $C$ must have  a parent that is anomalous and is either a root cause or is propagated from a root cause that is an ancestor of $C$.
		Similarly, if $C$ is not a root vertex and  belongs to an SCC $\mathbb{S}$ of size greater than $1$ and $Pa_{\mathcal{G}_{\mathbb A}}(\mathbb{S})\not\subset \mathbb{S}$, it also violates Assumption~\ref{assumption:one_intervention}. Indeed, at least one member of $\mathbb{S}$ has an anomalous parent that is not in $\mathbb{S}$ and that is either a root cause or is propagated from a root cause that is an ancestor of $C$; all other members have an anomalous parent in $\mathbb{S}$ that is propagated from a root cause that is an ancestor of $C$.
	\end{proof}

	% It is worth mentioning that opting for a T-FTCG instead of a T-SCG allows achieving the same outcome through a straightforward adaptation of Lemma~\ref{lemma:root_causes_forwards} to T-FTCGs (under Assumption~\ref{assumption:one_intervention} and without Lemma~\ref{lemma:root_causes_SCC}). Since both T-SCGs and T-FTCGs yield identical results, our focus in this paper has been on T-SCGs due to their greater comprehensiveness and smaller size.
	
	\subsection{Agent-based extension} %Uncertain propagation of anomalies}
\label{ssec:TRCAextentions}

% \begin{figure*}\centering
	%     \includegraphics[width = 1\textwidth, trim = 1.8cm 1.5cm 2.2cm 1.9cm, clip = TRUE]{Figures/pdf/TRCA_agent.pdf}
	%     \caption{TRCA-agent in practice. }%For an anomalous T-SCG learnt on offline and online datasets, an agent intervenes on root causes detected by T-RCA, here $Z$ and $X$. Looking at new online data, some vertices are still anomalous, $X$ and $W$, and T-RCA infers that $W$ is a root cause. Then the agent intervenes on $W$, and the system has no more issues.}
%     \label{fig:agent-based-trca}
% \end{figure*}

When Assumption~\ref{assumption:one_intervention} is violated, T-RCA struggles to distinguish root causes among anomalies. To address this, we propose T-RCA-agent, an agent-based extension. T-RCA-agent first runs T-RCA to identify the initial batch of root causes  associated with anomalies. An agent then rectifies incidents attributed to these root causes. If anomalies persist, T-RCA-agent reruns T-RCA on the remaining anomalous time series iteratively until no anomalies remain. Unlike T-RCA, T-RCA-agent guarantees root cause detection even when Assumption~\ref{assumption:one_intervention} is violated. Additionally, we theoretically determine the number of iterations needed for termination. %, as per the following proposition.
%
%\begin{restatable}{proposition}{mypropositionone}
\begin{proposition}
\label{proposition:optimal_agent}
Let $m$ be the maximum number of root causes on the same directed path in the T-SCG and not satisfying Assumption~\ref{assumption:one_intervention}.
Under Assumptions~\ref{assumption:cmc}, \ref{assumption:adjacency_faithfulness}, \ref{assumption:no_hidden_con}, \ref{assumption:consistency}, T-RCA-agent identifies all root causes after $m$ iterations.
%and the number of iterations needed to obtain all root causes is equal to
\end{proposition}
%\end{restatable}
%
\begin{proof}
Using Theorem \ref{theorem:TRCA}, only the initial root cause of each directed path is identified. Subsequently, after agent intervention, the number of root causes in each path decreases by at least 1 (although multiple interventions may occur simultaneously, T-RCA selects one randomly), maintaining a maximum of $m-1$ root causes per path. Anomalous variables in $\mathcal{D}_{\text{on}}$ are updated accordingly, potentially assigning another root cause as the path's root. By induction, after employing T-RCA $m$ times, all root causes are detected.
\end{proof}

% \subsection{Deployment and impact}
% \textcolor{blue}{T-RCA is currently being tested internally at EasyVista and will soon be available to customers.}

% \textcolor{blue}{The T-SCG for each customer's system is generated using two months of historical data from the EV Observe platform\footnote{\url{https://www.easyvista.com/products/ev-observe-proactive-monitoring}} and stored in the server database along with metric thresholds. Updates to customer systems require T-SCG updates. Integrated into EV Observe, the root cause analysis service allows users to request analysis during incidents. The server uses both online incident data and the stored T-SCG and thresholds to identify root causes with negligible execution time.}

%%%%%%%%%%%%%%%%%%%%%%%%%%%%%%%%%%%%%%%%%%%%%%%%%%%%%%%%%%

\section{Experiments}
\label{sec:experiments}

In this section, we first describe the experimental setup\footnote{Experiments on computational time  and robustness can be found in Appendix.}, then conduct a thorough analysis using simulated data generated from several models to evaluate our method, and finally we present an analysis using a real-world IT monitoring dataset\footnote{Python code of our method and experiments: \url{https://github.com/leizan/T-RCA}.}.

\subsection{Experimental setup}
\paragraph{T-RCA}
For causal discovery, we use the PCMCI algorithm~\citep{Runge_2019} in which the G-squared test, adapted for binary thresholding of time series, is employed to find conditional independencies.

\paragraph{Baselines}
We compare our methods with 8 other methods\footnote{We consider a version of EasyRCA and CIRCA where we learn the graph from normal offline data using PCMCI, denoted as EasyRCA$^*$~\citep{Assaad_2023} and CIRCA$^*$, respectively.}: RCD, CloudRanger, MicroCause, EasyRCA and EasyRCA$^*$, CIRCA and CIRCA$^{*}$ and AITIA-PM. For simulated data, results of EasyRCA and CIRCA are excluded as they require a causal graph input, which we assume unavailable\footnote{These results can be found in Appendix.}.
%For \textcolor{green}{EasyRCA$^{*}$, CIRCA$^{*}$}, CloudRanger, MicroCause and T-RCA, we use the PCMCI algorithm to deduce the graph from anomalous data in $\mathcal{D}_\text{off}$. 
For RCD, CIRCA$^{*}$, and AITIA-PM, we select the top two variables from their output ranking as inferred root causes, considering that each case has two genuine root causes.
%
% \subsubsection{Hyper-parameters}
%\paragraph{Hyper-parameters}
For T-RCA, EasyRCA$^{*}$, and MicroCause, we set the maximum lag $\gamma_{\max}$ to 1, and maintain a fixed significance level of 0.01 across all methods. For EasyRCA$^{*}$, CloudRanger, and MicroCause, a Fisher-z-test is utilized due to their use of raw (continuous) data as input.
Additionally, in the case of CloudRanger and MicroCause, the walk length is set to $1000$, and the backward step threshold is fixed to $0.1$. Default values are used for all other hyperparameters.
%Finally, other hyperparameters in CIRCA$^*$ and RCD are retained as specified in their packages.

\paragraph{Evaluation}
Accuracy of root cause detection is measured by the F1-score between the  real causes $\mathbb{C}$ and the inferred set $\mathbb{\hat C}$:
$$F1\text{-score}(\mathbb{C}, \mathbb{\hat C})=\frac{2\sum_{c\in \mathbb{C}}\mathbb{1}_{c\in\mathbb{\hat C}}}{2\sum_{c\in \mathbb{C}}\mathbb{1}_{c\in\mathbb{\hat C}} + \sum_{\hat c\in \mathbb{\hat C}}\mathbb{1}_{\hat c\not\in\mathbb{C}} + \sum_{c\in \mathbb{C}}\mathbb{1}_{c\not\in\mathbb{\hat C}}}.$$

\subsection{Simulated data}
\label{subsec:sim_data}

\begin{figure*}[t]
\centering\includegraphics[width = 0.8\textwidth, clip, trim=1cm 1.37cm 0cm 1.45cm]{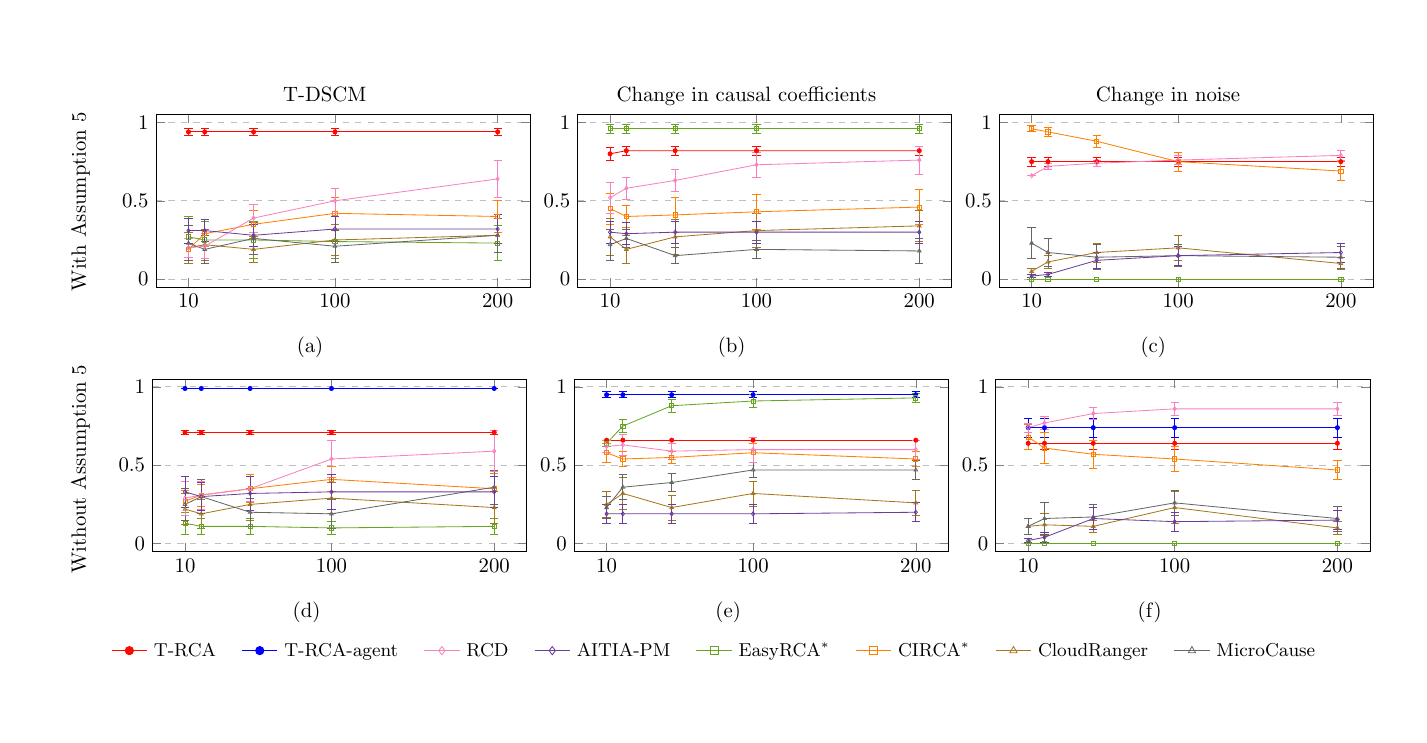}
%{Figures/pdf/logic_based.pdf}
\caption{
	Average F1-score and its variance across 50 simulations are depicted for simulated data. The length of $\mathcal{D}_{\text{on}}$ ranges from 10 to 200, generated from a T-DSCM (a and d), a DSCM with root causes experiencing changes in causal coefficients (b and e), and a DSCM with root causes undergoing changes in noise (c and f). Each model is assessed under two settings: one adhering to Assumption~\ref{assumption:one_intervention} (a, b, and c) and another that violates Assumption~\ref{assumption:one_intervention} (d, e, and f).
}
\label{fig:sim_data}
\end{figure*}

We explore three settings, each with distinct data generation processes. The first setting assesses method effectiveness by generating data from a T-DSCM, while the second and third settings evaluate method robustness using a different model. Within each setting, we examine two cases: one where Assumption~\ref{assumption:one_intervention} holds and another where it's violated. In the latter case, we introduce T-RCA-agent and set the number of iterations equal to the maximum number of genuine root causes in an active path. If T-RCA-agent identifies a genuine root cause, anomalous variables in $\mathcal{D}_{\text{on}}$ are adjusted accordingly, simulating actions taken by system engineers. Across all settings, $\mathcal{D}_{\text{off}}$ is 20,000 units long, while we vary the lengths of $\mathcal{D}_{\text{on}}$ in $\{10, 20, 50, 100, 200\}$.
% Additionally, results up to a length of 2,000 are presented in Appendix~\ref{app:exten_simu_data}.

\subsubsection{Threshold-based system}
\label{subsubsec:TDSCM}

The first data generation process is based on the T-DSCM outlined in Definition \ref{def:tdscm}. We randomly generate 50 T-SCGs, each with 6 vertices, a maximal degree between 4 and 5, and exactly one root vertex. All lags in the associated T-FTCGs are set to 1. Subsequently, we generate one dataset for each T-SCG.
Each time series is scaled to $\left[0, 1\right)$, with thresholds randomly selected from $U([0.7, 0.9])$. Every time series includes a self-cause, and with a probability of $0.3$, $\epsilon^y_t<1$, indicating anomalies may not always trigger in children. For $\mathcal{D}_{\text{off}}$, interventions are applied to each normal variable with a probability $\beta_t^y=0.1$. We ensure no time series remains anomalous for more than 5 consecutive time points. In $\mathcal{D}_{\text{on}}$, if Assumption \ref{assumption:one_intervention} holds, two vertices on the same active path are randomly chosen; otherwise, two vertices on different active paths are selected.

Results depicting the means and variances of the F1-score for each method are presented in Figure~\ref{fig:sim_data}(a) and (d), with and without Assumption \ref{assumption:one_intervention} respectively, where the values of thresholds are presumed to be known. 
% Robustness against mispecification of the thresholds has been tested and is illustrated in Appendix~\ref{app:robu_monitoring_data}. 
Notably, when Assumption~\ref{assumption:one_intervention} holds, T-RCA consistently outperforms other methods, exhibiting low variance. The performance of T-RCA remains stable since it only requires knowledge of the anomalous variables in $\mathcal{D}_{\text{on}}$. However, in the other case, the performance of T-RCA declines due to the violation of Assumption~\ref{assumption:one_intervention}, while T-RCA-agent demonstrates superior performance. All other methods have very low performances. Note that, as the length of $\mathcal{D}_{\text{on}}$ increases, the performance of RCD improves.

\begin{figure*}
\includegraphics[scale=0.8,trim = 0.5cm 3.55cm 1.5cm 0.5cm, clip=TRUE]{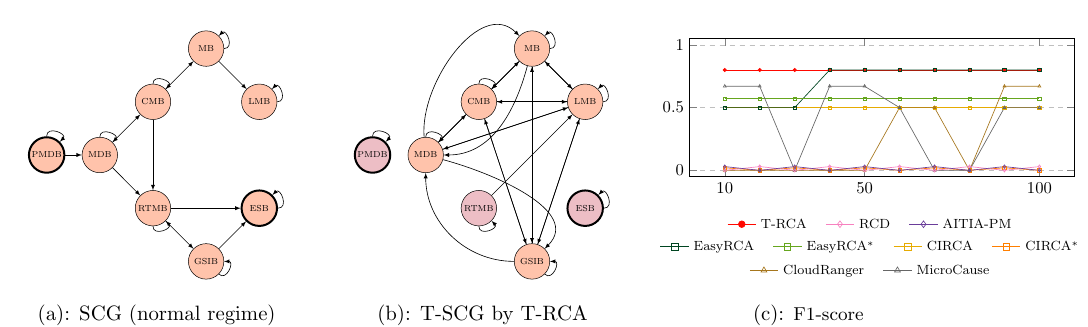}
\caption{Real IT monitoring data: (a) the SCG provided by the experts, on the normal regime, where root causes correspond to the vertices with thick borders (PMDB and ESB); (b) the T-SCG learned by T-RCA, where inferred root causes correspond to purple vertices (PMDB and ESB); (c) F1-score for the IT monitoring data, varying the lengths of $\mathcal{D}_{\text{on}}$ from 10 to 100. }
\label{fig:easyvista}
\end{figure*}

\subsubsection{Non-threshold based system with changes in causal coefficients}
\label{subsubsec:change_coefficients}

% $$ %\begin{equation}
%  y_t = \sum_{x_{t-1} \in Pa_{\mathcal{G}_{ft}}(y_t)}a x_{t-1} + 0.1\xi_t^{y},
% \label{eq:continuous_scm}
% $$ %\end{equation}

To evaluate the robustness of our method, we utilize a dataset introduced by \citet{Assaad_2023} comprising 50 different Acyclic Summary Causal Graphs with self-causes. All lags are set to 1. $\mathcal{D}_{\text{off}}$ is generated using the DSCM as $ 
y_t = \sum_{x_{t-1} \in Pa_{\mathcal{G}_{ft}}(y_t)}a x_{t-1} + 0.1\xi_t^{y},
$
with $a \sim U([0.1, 1])$, $\xi_t^{y} \sim N(0,1)$, $y_t$ denotes the value of the vertex $y$ at time $t$, $Pa_{\mathcal{G}_{ft}}(y_t)$ denotes the direct parents of $y_t$ in the FTCG. For $\mathcal{D}_{\text{on}}$, similar to the previous setting, two different strategies are used to randomly select two vertices in each graph for intervention, which are considered as the root causes. The intervention changes the coefficients from all parents of the intervened variable by resampling them from $ U([0.1, 1])$. The effect of each intervention propagates through the generating process to all the descendants of the intervened vertex.
Thresholds for each time series are chosen empirically such that each variable in $\mathcal{D}_{\text{off}}$ contains $90\%$ of data below the threshold and $10\%$ above it. In this setting, T-RCA and T-RCA-agent utilize thresholds to learn the T-SCG from $\mathcal{D}_{\text{off}}$. The selected root causes and their descendants in the graph are considered as anomalous variables, and this information is utilized by our proposed method, EasyRCA$^*$, CloudRanger, and MicroCause.

Results depicting the means and variances of the F1-score for each method are presented in Figure~\ref{fig:sim_data}  (b) and (e) for both cases. When Assumption~\ref{assumption:one_intervention} holds, EasyRCA$^*$ performs well with low variance, as the data generating process aligns with its settings. The performance of T-RCA follows EasyRCA$^*$, outperforming other methods. In the other case, T-RCA suffers, while the performance of EasyRCA$^*$ increases along with the expansion of the length of $\mathcal{D}_{\text{on}}$. Notably, T-RCA-agent exhibits a comparable performance with EasyRCA$^*$, and even better for a small length of $\mathcal{D}_{\text{on}}$. %Notably, T-RCA-agent outperforms EasyRCA$^*$ consistently, particularly when the length of $\mathcal{D}_{\text{on}}$ is small.  %Among the other methods, CIRCA$^*$ exhibits better performance in both cases. 

%\textbf{Additionally}, in the Appendix, we investigate the performance of our proposed method when the thresholds are not selected accurately in the previously mentioned three cases.

\subsubsection{Non-threshold based system with changes in noise}
\label{subsubsec:change_noise}
Finally, we utilize a dataset simulating a microservice architecture from the DoWhy package\footnote{\url{https://www.pywhy.org/dowhy/v0.11.1/example_notebooks/gcm_rca_microservice_architecture.html}}. The graph consists of 11 vertices, and a lag of 1 is assumed between each pair of vertices to simulate time series data. Similar to the previous setting, two different strategies are used to randomly select two vertices in the graph for intervention. We generate 50 datasets for each case. Interventions are applied following Scenario 3 in the provided link, which involves shifting the value by a constant. Thresholds for each time series are determined by using the empirical strategy described earlier.

Results depicting the means and variances of the F1-score for each method are presented in Figure~\ref{fig:sim_data}  (c) and (f) for both cases. When Assumption~\ref{assumption:one_intervention} holds, CIRCA$^*$ exhibits good performance with a small size of $\mathcal{D}_{\text{on}}$. However, its performance declines as $\mathcal{D}_{\text{on}}$ increases. T-RCA maintains performance comparable to RCD. When Assumption~\ref{assumption:one_intervention} is violated, RCD outperforms other methods once $\mathcal{D}_{\text{on}}$ exceeds 20, followed by T-RCA-agent, T-RCA, and CIRCA$^*$. These methods clearly outperform others in both cases.

\subsection{Real IT monitoring data}
The dataset, provided by EasyVista and introduced in \cite{Assaad_2023}, consists of eight time series collected from an IT monitoring system with a one-minute sampling rate, as described in \cite{Assaad_2023, Ait_Bachir_2023}.
% \begin{itemize}
%  \item PMDB: extraction of some information about the messages received by the Storm ingestion system;
%  \item MDB: activity of a process that orient messages to other process with respect to different types of messages; 
%  \item CMB: activity of extraction of metrics from messages; 
%  \item MB: activity of insertion of data in a database; 
%  \item LMB: reflects the updates the last values of metrics in Cassandra; 
%  \item RTMB: activity of searching to merge of data with information coming from the check message bolt; 
%  \item GSIB: activity of insertion of historical status in database;
%  \item ESB: activity of writing data in Elasticsearch.
% \end{itemize}

For T-RCA and AITIA-PM methods, which rely on anomalies in historical data for root cause detection, consider all data preceding the onset of anomalies as $\mathcal{D}_\text{off}$, comprising more than 40,000 data points. However, for methods needing only normal data as $\mathcal{D}_\text{off}$, we use 1,000 data points from the normal part before anomalies. We vary the length of $\mathcal{D}_\text{on}$ from 10 to 100 with a step of 10. Thresholds for each time series are selected such that each variable in $\mathcal{D}_{\text{off}}$ contains $90\%$ of data below the threshold and $10\%$ above it.
% with robustness tested in Appendix~\ref{app:robu_monitoring_data}. 
Each time series is anomalous within $\mathcal{D}_\text{on}$. For CIRCA, we consider only one ultra-metric, \textit{Saturation}. Due to the inability to simulate system engineer interventions, T-RCA-agent is not part of the comparison.

The true T-SCG associated with this system is unknown, but EasyVista’s system experts have described the summary causal graph in the normal regime (when there are no anomalies), which is provided in Figure~\ref{fig:easyvista}(a). Thus, for methods needing a graph, such as CIRCA and EasyRCA, the summary causal graph in the normal regime serves as the input graph. PMDB and ESB are expected to be the root causes of the anomalies. 
The T-SCG learned by T-RCA is given in Figure~\ref{fig:easyvista}(b). The graph is denser than the SCG stated in Figure~\ref{fig:easyvista}(a), but they do not encode the same information. Most importantly, the genuine root causes are discovered.

The F1-score of each method is presented in Figure~\ref{fig:easyvista}(c). T-RCA consistently performs well, achieving an F1-score of 0.8, similar to EasyRCA which knows the SCG in the normal regime, when the length of $\mathcal{D}_\text{on}$ exceeds 30. CIRCA and EasyRCA$^*$ exhibit reasonable performance. Meanwhile, the result also exhibits the importance of the correct causal graph for the method CIRCA and EasyRCA. The performance of CloudRanger and MicroCause varies considerably, influenced by random mechanisms in their respective methods.

% T-RCA-PC and EasyRCA inferred 3 roots causes, PMDB, RTMB and ESB. T-RCA-AITIA inferred all vertices as the root cause. EasyRCA$^*$ inferred 5 roots causes, MDB, RTMB, PMDB, LMB, and ESB. CIRCA inferred PMDB and RTMB as the root causes. MicroCause inferred CMB as the root cause. CloudRanger inferred GSIB and LMB as the root causes.

%%%%%%%%%%%%%%%%%%%%%%%%%%%%%%%%%%%%%%%%%%%%%%%%%%%%%%%%%%

\section{Conclusion}
\label{sec:discussion}
We introduced a novel structural causal model tailored for representing threshold-based IT systems and presented the T-RCA algorithm for detecting root causes of anomalies within such systems. In its basic form, T-RCA assumes the absence of more than one root cause in an active path aligned with the derived causal graph. Additionally, we introduced an optimized agent-based extension of T-RCA, relaxing this assumption while necessitating minimal system interventions. Our experiments showcased the superiority of our methods, particularly on data generated from the specified structural causal model, but also on data from alternative models.

For future work, it would be interesting to consider instantaneous relations. In this case,  having richer graphs, which are inherently acyclic, can offer advantages, especially when anomalies form a cycle in the T-SCG. 
Additionally, relaxing the assumption of joint independence among variables in $\mathbb{U}_t$ and $\mathbb{I}_t$ could provide further insights. Finally, exploring solutions that alleviate Assumption \ref{assumption:one_intervention} without requiring external actions is of interest.

\section*{Acknowledgements}
This work was partially supported by EasyVista, by MIAI@Grenoble Alpes (ANR-19-P3IA-0003), and by the CIPHOD project (ANR-23-CPJ1-0212-01).
%We thank Christophe de Bignicourt and Hosein Mohanna from EasyVista for several discussions about threshold-based IT systems.

% \bibliographystyle{alpha}
\bibliographystyle{abbrvnat}
\bibliography{references}

% \clearpage
% \onecolumn
\appendix
% \section{Proofs}
% \label{appendix:proofs}

% \mylemmaone*
% \begin{proof}
	%     1
	% \end{proof}

%\mylemmaone*

%\mylemmatwo*

%\mylemmathree*

%\mytheoremone*

% \mypropositionone*
% \begin{proof}
	%    Using Theorem \ref{theorem:TRCA}, only the root cause at the beginning of each directed path is identified. Subsequently, after the intervention of an agent, the number of root causes in each directed path decreases by at least 1 (even though multiple interventions may occur simultaneously, T-RCA selects one randomly), maintaining the maximum number of root causes on a directed path at $m-1$. Anomalous variables in $\mathcal{D}_{\text{on}}$ are updated accordingly, and another root cause may become the root in this directed path. By induction, after employing T-RCA $m$ times, all root causes are detected.
	% \end{proof}

\section{Pseudo-code of T-RCA}
The pseudo-code of T-RCA is presented in Algorithm~\ref{algo:TRCA}. 
The algorithm outlined begins at line 1 by converting d-dimensional observational time series from two datasets, $\mathcal{D}_{off}$ and $\mathcal{D}_{on}$, into binary thresholded series using predefined thresholds ($\mathbb{R}^d$). 
In line 2, using an event-based causal discovery algorithm on $\mathcal{D}_{off}$, a T-SCG is constructed. 
In line 3 and 4, from $\mathcal{D}_{on}$, anomalous vertices are identified and their times of appearance are mapped, employing Definitions 5 and 7 respectively. 
In line 5, the algorithm initializes an empty set for root causes, $\mathbb{\hat{C}}$. 
In lines 6-11, the algorithm analyzes each SCC within each anomalous subgraph. If all parents of an SCC are contained within the same SCC and it consists of a single vertex, that vertex is added to $\mathbb{\hat{C}}$ as per Lemma 2. For SCCs containing multiple vertices, the vertex with the earliest appearance time of anomaly is selected for $\mathbb{\hat{C}}$.

\begin{algorithm}[h]
	\caption{T-RCA}
	\label{algo:TRCA}
	\DontPrintSemicolon
	\KwData{Two datasets $\mathcal{D}_{\text{off}}$ and $\mathcal{D}_{\text{on}}$ of $d$-dimensional observational time series $\mathbb{T}$, maximum lag $\gamma_{\max}$, set of thresholds for each time series $\mathbb{R}^d$}
	\KwResult{Set of root causes $\mathbb{\hat C}$}
	Provide the binary thresholded time series according to thresholds $\mathbb{R}^d$\\
	Discover the T-SCG $\mathcal{\hat G}=(\mathbb{V}^r, \mathbb{E}^r)$ using an event-based causal discovery algorithm on $\mathcal{D}_{\text{off}}$ and $\mathbb{R}^d$ \tcp*{Lemma~\ref{lemma:identifiability_of_TFTCG}}
	%Deduce T-SCG $\mathcal{\hat G}=(\mathbb{V}^r, \mathbb{E}^r)$ from $\mathcal{\hat G}_{\text{ft}}$ using Definition~\ref{def:t-SCG}\;
	Deduce the set of anomalous vertices $\mathbb{A}\subseteq \mathbb{V}^r$ from $\mathcal{D}_{\text{on}}$ using Definition~\ref{def:anomaly_TDSCM_version}\;
	Deduce the mapping of appearance time of anomalies $\mathbb{\tau}$ from $\mathcal{\hat G}$ and $\mathcal{D}_{\text{on}}$  using Definition~\ref{def:appearance_time}\; 
	$\mathbb{\hat C} = \emptyset$\;
	\ForEach{SCC $\mathbb{S}\in \mathcal{\hat 
			G}_{\mathbb{A}}$}{
		\If{$Pa_{\mathcal{\hat 
					G}_{\mathbb{A}}}(\mathbb{S}) \subseteq \mathbb{S}$}{
			\If{size$(\mathbb{S})=1$}{
				$\mathbb{\hat C}$= $\mathbb{\hat C} \cup \mathbb{S}$ \tcp*{Lemma~\ref{lemma:root_causes_forwards}}
			}
			\Else{
				$\mathbb{\hat C}$= $\mathbb{\hat C} \cup \{X^{\ge r_x}\}$ such that $X^{\ge r_x}\in \mathbb{S}$ and $\tau(X^{\ge r_x})<\tau(Y^{\ge r_y})$ for all $\qquad$ $Y^{\ge r_y} \in \mathbb{S}\backslash\{X^{\ge r_x}\}$ \tcp*{Lemma~\ref{lemma:root_causes_SCC}}
			}
		}
	}
\end{algorithm}

\section{Examples of Insufficiency of correlation and time}
% \begin{restatable}{property}{mypropertyone}
	% \label{theorem:potential_cause_SCG}
	% If all assumptions are satisfied, the detection of \textit{prima facie causes} in the form of T-SCG is not sufficient to identify all root causes.
	% \end{restatable}
\label{app:ex}

It might appear that under Assumption~\ref{assumption:one_intervention}, the time of anomaly and/or the association (possibly conditional) between anomalies might be sufficient to detect root causes.
We provide two examples showing that this is, in fact, not true,  highlighting the crucial role of the causal discovery step.

\begin{example}[Insufficiency of time and dependence]
	Consider a scenario involving three time series, where the underlying T-SCG among them are depicted in Figure~\ref{fig:counter_examples} (a). The lag between $X$ and $Y$ is 1, while the lag between $X$ and $Z$ is 2. Assuming  $\epsilon^y_t<1$ and ignoring self-causes, multiple interventions are applied to $X$ at different time steps to induce anomalies.
	
	When identifying the root cause(s) of $Z^{\ge r_z}=1$ (denoted as $Z^{\ge r_z}_{=1}$) solely based on temporal information, both $X^{\ge r_x}_{=1}$ and $Y^{\ge r_y}_{=1}$ are included, as they occur before $Z^{\ge r_z}_{=1}$. Similarly, solely relying on dependence also points to both $X^{\ge r_x}_{=1}$ and $Y^{\ge r_y}_{=1}$ as root causes, given their connections to $Z^{\ge r_z}_{=1}$. Even when considering temporal and dependence information together, $Y^{\ge r_y}_{=1}$ remains identified as the root cause. %However, contrary to these conclusions, the genuine root cause is $X^{\ge r_x}_{=1}$.
	
\end{example}

% \begin{restatable}{property}{mypropertytwo}
	% \label{theorem:epsilon_SCG}
	% If all assumptions are satisfied, RAitia-2011 in the form of T-SCG is not sufficient to identify all root causes.
	% \end{restatable}

\begin{figure}[!ht]
	\centering
	\begin{subfigure}{0.45\textwidth}
		\centering
		% \begin{tikzpicture}[{black, circle, draw, inner sep=0}]
			%   \tikzset{nodes={draw,rounded corners},minimum height=0.8cm,minimum width=0.8cm, font=\footnotesize}
			%   \tikzset{latent/.append style={fill=gray!30}}
			
			%   \node [fill=easyorange!40] (Z) at (1,1) {$Z^{\ge r_z}$} ;
			%   \node [fill=easyorange!40] (Y) at (1,-1) {$Y^{\ge r_y}$};
			%   \node [fill=easyorange!40, ultra thick] (X) at (0,0) {$X^{\ge r_x}$};
			
			%   % \node[draw=none, minimum width = 0.5cm,minimum height = 0.5cm] (iX) at (-1.4, 0.8) {};
			%   % \draw [->, line join=round, color = red, decorate, decoration={zigzag, segment length=4, amplitude=.9,post=lineto, post length=2pt }](iX) -- (X);
			
			%   \draw[->,>=latex] (X) -- (Y);
			%   \draw[->,>=latex] (X) -- (Z);

			% \end{tikzpicture}
		\includegraphics[trim = 0 1cm 0 1cm, clip=TRUE]{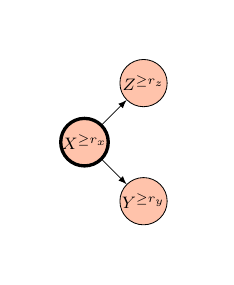} 
		\label{fig:counter_example_1}
		\caption{Illustration for example 1}
	\end{subfigure}
	\hfill 
	\begin{subfigure}{0.45\textwidth}
		\centering
		%     \begin{tikzpicture}[{black, circle, draw, inner sep=0}]
			%       \tikzset{nodes={draw,rounded corners},minimum height=0.8cm,minimum width=0.8cm, font=\footnotesize}
			%       \tikzset{latent/.append style={fill=gray!30}}
			
			% \node [fill=easyorange!40] (X) at (0,0) {$X^{\ge r_x}$} ;
			% \node [fill=easyorange!40, ultra thick] (Z) at (1,1) {$Z^{\ge r_z}$};
			% \node [fill=easyorange!40] (Y) at (2,0) {$Y^{\ge r_y}$};
			% \node [fill=easyorange!40, ultra thick] (W) at (1,-1) {$W^{\ge r_w}$};
			% \draw[->,>=latex] (Y) to [out=0,in=45, looseness=2] (Y);
			% \draw[->,>=latex] (Z) to [out=0,in=45, looseness=2] (Z);
			% \draw[->,>=latex] (X) to [out=180,in=135, looseness=2] (X);
			% \draw[->,>=latex] (W) to [out=0,in=-45, looseness=2] (W);
			
			%       %   \node[draw=none, minimum width = 0.5cm,minimum height = 0.5cm] (iZ) at (-0.4, 1.8) {};
			%       % \draw [->, line join=round, color = red, decorate, decoration={zigzag, segment length=4, amplitude=.9,post=lineto, post length=2pt }](iZ) -- (Z);
			
			%       %   \node[draw=none, minimum width = 0.5cm,minimum height = 0.5cm] (iW) at (-0.4, -0.1) {};
			%       % \draw [->, line join=round, color = red, decorate, decoration={zigzag, segment length=4, amplitude=.9,post=lineto, post length=2pt }](iW) -- (W);
			
			% \draw[->,>=latex] (Z) -- (X);
			% \draw[->,>=latex] (Z) -- (Y);
			
			% \draw[->,>=latex] (W) -- (X);
			% \draw[->,>=latex] (W) -- (Y);
			% \end{tikzpicture}
		\includegraphics[trim = 0 1cm 0 1cm, clip=TRUE]{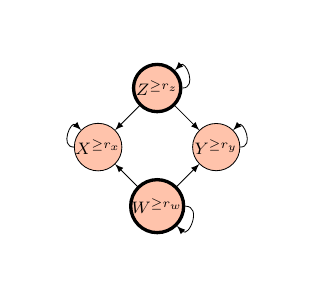} 
		\caption{Illustration for example 2}
		\label{fig:counter_example_2}
	\end{subfigure}
	\caption{(a) demonstrates a scenario where time and dependence fail to detect all root causes. (b) demonstrates a scenario where time and conditional dependence on a single variable fail to detect all root causes.}
	\label{fig:counter_examples}
\end{figure}

\begin{example}[Insufficiency of time and conditional dependence on a single variable]
	
	Consider another scenario involving four time series, where the underlying T-SCG among them is depicted in Figure~\ref{fig:counter_examples} (b). The lag between $W$, $X$ and $Z$, $X$ is 1, while the lag between $W$, $Y$ and $Z$, $Y$ is 2. Assuming $\epsilon^y_t<1$ ($i.e.$, uncertainty in anomaly propagation) and ignoring self-causes,  multiple interventions are applied independently on $W$ and $Z$ at different time steps to induce anomalies. Here, we aim to detect all root causes of $Y^{\ge r_y}=1$ (denoted as $Y^{\ge r_y}_{=1}$).
	
	Firstly, using temporal order, $W^{\ge r_w}_{=1}$, $X^{\ge r_x}_{=1}$ and $Z^{\ge r_z}_{=1}$ are considered as potential causes of $Y^{\ge r_y}_{=1}$ because they precede it. 
	
	% Even employing a special measure that sums over multiple conditional dependences doesn't guarantee accurate root cause detection. This is because, within each term, we only condition on one variable, preserving the dependence between $X^{\ge r_x}{=1}$ and $Y^{\ge r_y}_{=1}$.
	
	Then, using dependence conditioning on a single variable cannot guarantee the detection of root causes. For example, when conditioning on $W^{\ge r_w}_{=1}$ (resp. $Z^{\ge r_z}_{=1}$), $X^{\ge r_x}_{=1}$ becomes dependent on $Y^{\ge r_y}_{=1}$, because of $Z^{\ge r_z}_{=1}$ (resp. $W^{\ge r_w}_{=1}$). Consequently, $X^{\ge r_x}_{=1}$ is incorrectly considered a root cause. 
	
	Even using a special measure that sums over multiple conditional dependence,  we will not be able to guarantee the detection of root causes: as previously, between $X^{\ge r_x}_{=1}$ and $Y^{\ge r_y}_{=1}$, even though we do a sum of multiple terms of conditional dependence, but within each term, we only condition on one variable, which will not break the dependence between these two variables.
	
	Furthermore, even if we search for the highest conditional dependence, as  done in AITIA-PM~\citep{van2021root}, there is no guarantee to find all root causes. For instance, consider again the T-SCG in Figure~\ref{fig:counter_examples} (b), but now assume $\epsilon^y_t=1$. Then, the special measure (used by AITIA-PM) that sums over multiple conditional dependence between $Y^{\ge r_y}_{=1}$ and $X^{\ge r_x}_{=1}$ is :
	{\small
		\begin{align*}    
			\{ &\Pr(Y^{\ge r_y}_{=1} |X^{\ge r_x}_{=1} \wedge W^{\ge r_w}_{=1}) -\Pr(Y^{\ge r_y}_{=1} |X^{\ge r_x}_{=0} \wedge W^{\ge r_w}_{=1}) \\
			&+ \Pr(Y^{\ge r_y}_{=1} |X^{\ge r_x}_{=1}\wedge Z^{\ge r_z}_{=1}) -\Pr(Y^{\ge r_y}_{=1} |X^{\ge r_x}_{=0} \wedge Z^{\ge r_z}_{=1})\}
			% & = \Pr(Y^{\ge r_y}_{=1} |X^{\ge r_x}_{=1} \wedge W^{\ge r_w}_{=1}) + \Pr(Y^{\ge r_y}_{=1} |X^{\ge r_x}_{=1}\wedge Z^{\ge r_z}_{=1}) \\
			= 2
		\end{align*}
	}
	which is higher than the same special measure between $Y^{\ge r_y}_{=1}$ and $Z^{\ge r_z}_{=1}$ equal to:
	{\small 
		\begin{align*}
			\{ &\Pr(Y^{\ge r_y}_{=1} |W^{\ge r_w}_{=1} \wedge Z^{\ge r_z}_{=1}) -\Pr(Y^{\ge r_y}_{=1} |W^{\ge r_w}_{=0} \wedge Z^{\ge r_z}_{=1}) \\
			&+ \Pr(Y^{\ge r_y}_{=1} |W^{\ge r_w}_{=1}\wedge X^{\ge r_x}_{=1}) -\Pr(Y^{\ge r_y}_{=1} |W^{\ge r_w}_{=0} \wedge X^{\ge r_x}_{=1})\}
			= 0.
		\end{align*}
	}

\end{example}

% \section{Causal discovery}
% \label{appendix:SOTA-cd}
% The task of inferring causal relations from observational time series holds significant importance across various scientific domains \citep{Spirtes_2000}. This task lies at the core of causal discovery \citep{Assaad_2022}. Methods within this field allow us to infer causal relations from observational data, particularly in cases where interventions like randomized trials are unfeasible, ethically questionable, or economically prohibitive. 

% Constraint-based methods \citep{Spirtes_2000} rank among the most popular strategies for discovering causal graphs. In the case of \textcolor{red}{causal sufficiency}, these methods usually rely on the PC-algorithm~\citep{Spirtes_2000}. PCMCI~\citep{Runge_2019, Runge_2020} is a variant of the PC-algorithm tailored for time series data, adept at identifying window causal graphs. On the other hand, PCGCE~\citep{Assaad_2022PCGCE} and PCTMI~\citep{Assaad_2022PCTMI} independently modify the PC-algorithm to discover extended summary causal graphs and summary causal graphs. These methods operate on the premise of \textcolor{red}{faithfulness}, asserting that all conditional independencies stem from the \textcolor{red}{causal Markov condition}.

% Logic-based methods ...

% There are additional families of methods that we believe fall outside the scope of this paper, as, to our knowledge, they are not employed in any root cause analysis methods. Further information about these families of methods can be found in \cite{Assaad_2022}.

\section{Complementary experiments}
\label{app:exp}

\subsection{Extension of Section~\ref{subsec:sim_data} - varying the length of $\mathcal{D}_{on}$}
\label{app:exten_simu_data}

In Figure~\ref{fig:sim_data_complete}, we present the results of the means and variances of the F1-score for each method, considering a larger length of $\mathcal{D}_{\text{on}}$ up to 2,000 samples. These results align with the three settings delineated in Section~\ref{subsec:sim_data}. In addition, in this analysis, we include the results of EasyRCA and CIRCA, which require a causal graph as input.  In general, EasyRCA shows comparable performance to T-RCA but falls slightly behind T-RCA-agent. The good performance of EasyRCA is not surprising since it relies on an input graph. Nevertheless, it is noteworthy that T-RCA consistently surpasses CIRCA in performance, which was unexpected.

\begin{figure}
	\centering
	\includegraphics[width = 0.9\textwidth, trim = 0 1.4cm 0 1.5cm, clip = TRUE]{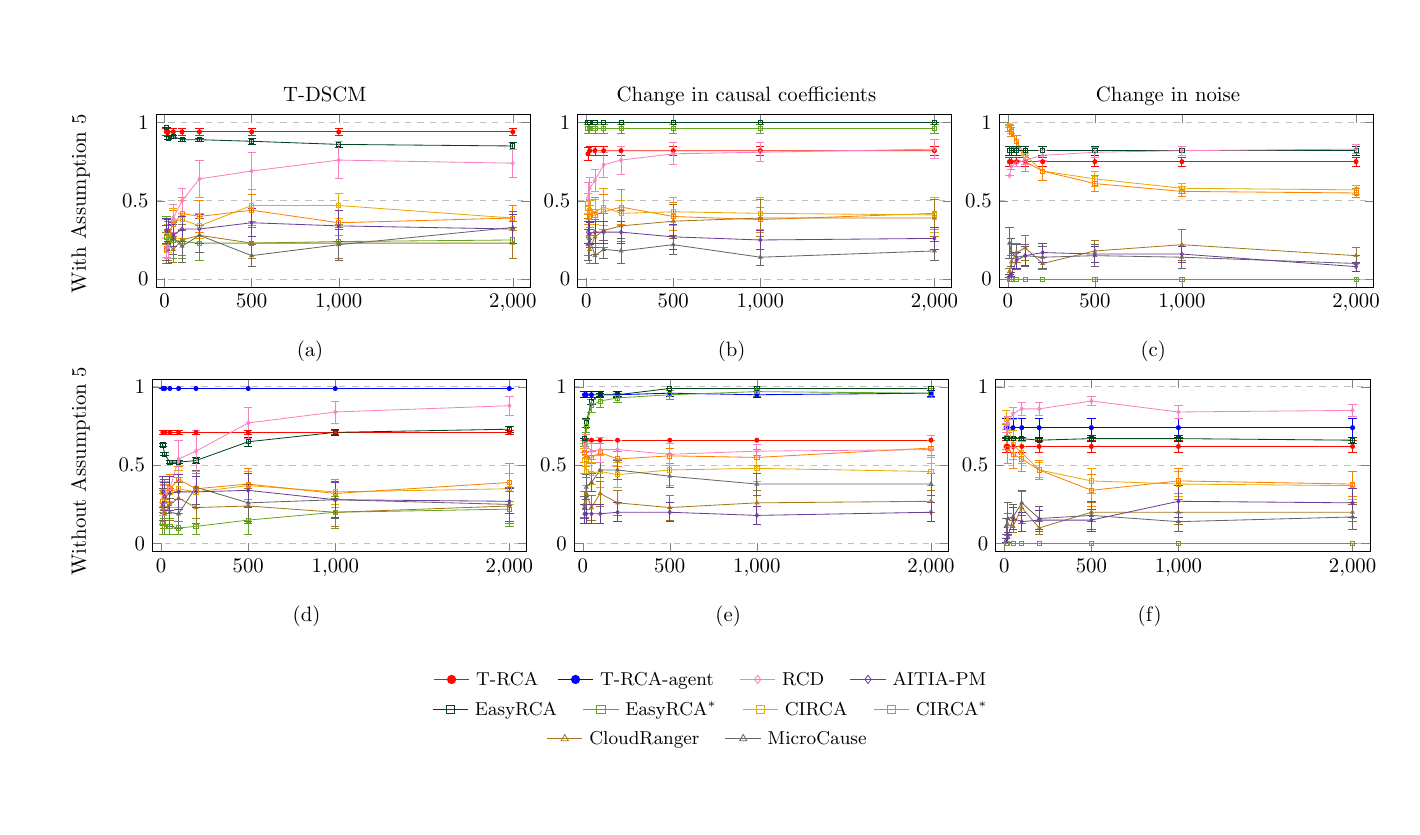}
	\caption{Mean of F1-score averaged over 50 simulations, and associated variance, for the three data generating processes, the lengths of $\mathcal{D}_{\text{on}}$ varying from 10 to 2,000.}
	\label{fig:sim_data_complete}
\end{figure}

\subsection{Extension of Section~\ref{subsubsec:TDSCM} - new data generating process}
\label{app:exten_TDSCM}

\begin{figure}
	\centering
	\includegraphics[width = 0.9\textwidth, trim = 0 1.4cm 0 1.5cm, clip=TRUE]{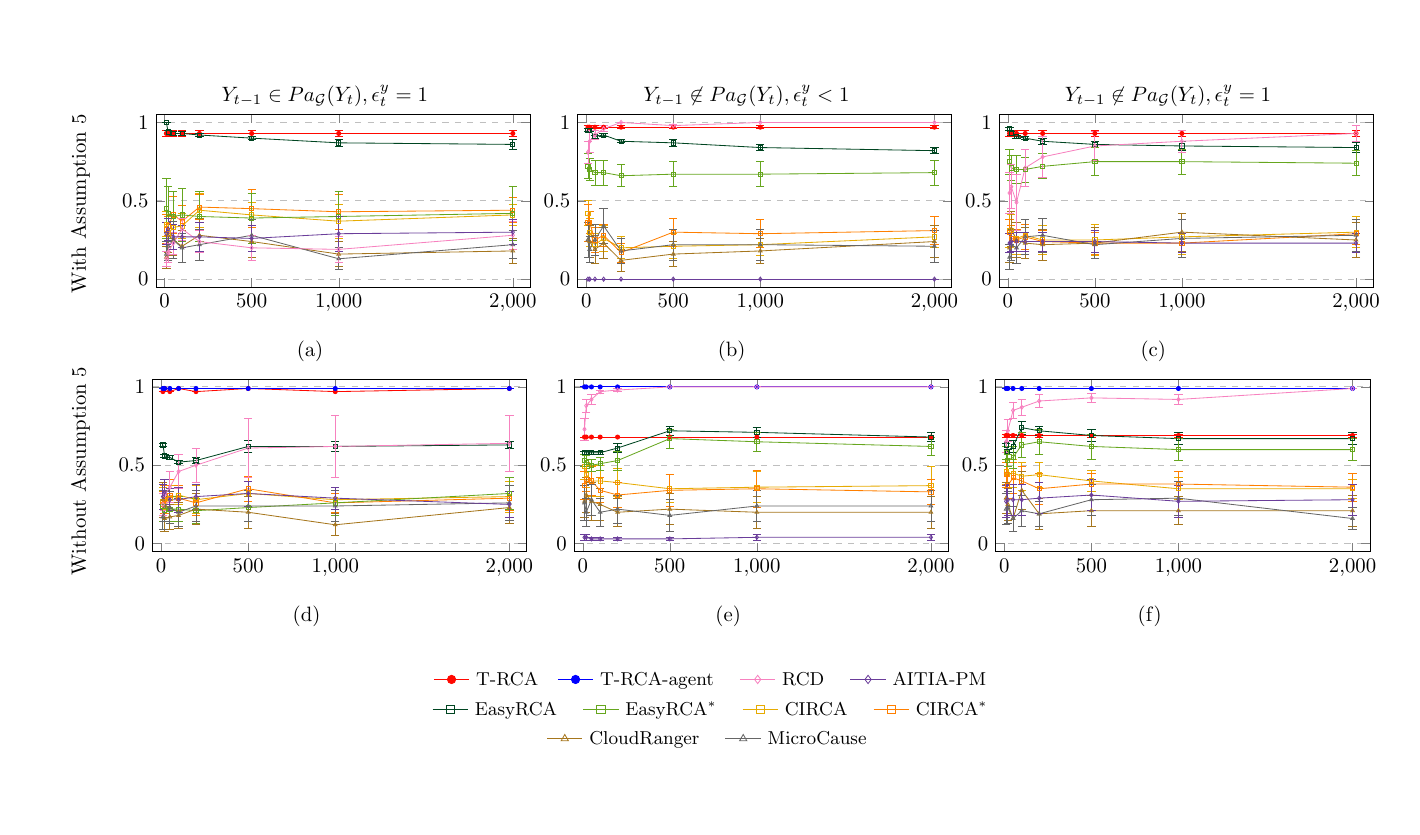}
	\caption{Mean of F1-score averaged over 50 simulations, and associated variance, for other three settings based on T-DSCM, the lengths of $\mathcal{D}_{\text{on}}$ varying from 10 to 2,000.}
	\label{fig:TDSCM_complete}
\end{figure}

In addition to the setting discussed in Section~\ref{subsubsec:TDSCM}, we explore three other settings based on T-DSCM, varying the lengths of the online set, $\mathcal{D}_\text{on}$, from 10 to 2,000 samples. Firstly, we distinguish if self-causes are considered for each vertex in the T-SCG or not (denoted as $Y_{t-1}\in Pa_{\mathcal{G}}(Y_t)$ or $Y_{t-1}\not\in Pa_{\mathcal{G}}(Y_t)$). Secondly, we distinguish $\epsilon^y_t=1$ ($i.e.$, certainty in anomaly propagation, referred to as $\epsilon^y_t=1$) and $\epsilon^y_t<1$ ($i.e.$, uncertainty in anomaly propagation, referred to as $\epsilon^y_t=1-0.3^{|Pa_{\mathcal{G}_{\text{ft}}}(Y_t)\cap \mathbb{A}|}$). The results are presented in Figure~\ref{fig:TDSCM_complete} which shows that when Assumption~\ref{assumption:one_intervention} is satisfied, overall T-RCA performs best followed by EasyRCA and RCD. When Assumption~\ref{assumption:one_intervention} is violated, T-RCA-agent performs best followed by T-RCA when $y_{t-1}\in Pa_{\mathcal{G}_{ft}}(y_t)$ and $\epsilon^y_t=1$ and by RCD in the other cases.

\subsection{Execution time analysis}
\label{app:time_analysis_simu_data}

Based on the setting discussed in Section~\ref{subsubsec:TDSCM}, where Assumption~\ref{assumption:one_intervention} holds, the execution time of each method is analyzed with $\mathcal{D}_{\text{on}}$ lengths varying within the range of [50, 500, 2000]. Additionally, the online part of T-RCA (T-RCA$_{online}$) and the offline part of T-RCA (T-RCA$_{offline}$) are analyzed separately. 

The mean of the logarithm of execution time (plus 1) in seconds, averaged over 50 simulations, along with the associated variance for each method, is presented in Figure~\ref{fig:execution_time_analysis}. The results indicate that the execution time of the offline part of T-RCA (T-RCA$_{offline}$) is the fastest among methods that need to reconstruct a graph based on $\mathcal{D}_{\text{off}}$. In contrast, the execution time of the online part of T-RCA (T-RCA$_{online}$) is negligible.

\begin{figure}[!ht]
	\centering
	\includegraphics[width= 0.7\textwidth, trim = 1.cm 1cm 0.5cm 0.5cm, clip=TRUE]{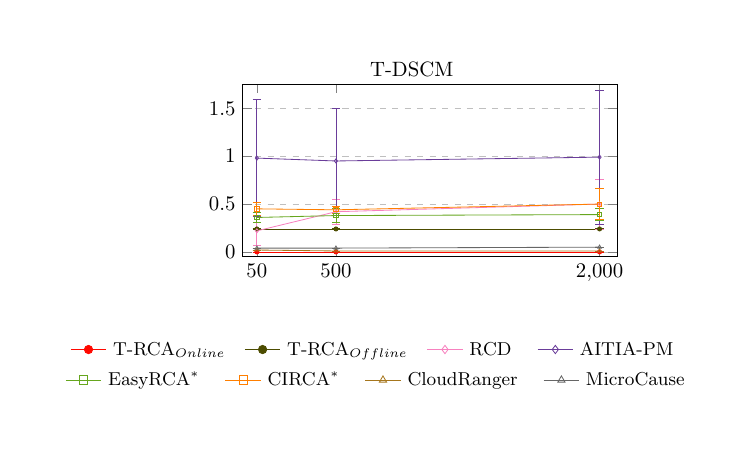}
	\caption{Mean of the logarithm of execution time (plus 1) in seconds, averaged over 50 simulations, along with the associated variance, is presented for the data-generating processes based on T-DSCM under Assumption~\ref{assumption:one_intervention}. The analysis separately considers the online part of T-RCA (T-RCA$_{online}$) and the offline part of T-RCA (T-RCA$_{offline}$), with $\mathcal{D}_{\text{on}}$ lengths varying within the range of [50, 500, 2000].}
	\label{fig:execution_time_analysis}
\end{figure}

\subsection{Robustness according to the choice of the threshold- equal thresholds}\label{app:rob} 
\subsubsection{Simulated data}
\label{app:robu_simu_data}
For our proposed methods, T-RCA and T-RCA-agent, thresholds are required for binarizing each time series. Here, we aim to investigate how changes in thresholds impact the performance of our proposed methods across the same cases outlined in Section~\ref{subsec:sim_data}. We vary the threshold for each variable by controlling the proportion of data below this threshold in $\mathcal{D}_{\text{off}}$ from 0.8 to 0.98 in steps of 0.02. Specifically, a proportion of 0.8 implies that the threshold for each time series is chosen to ensure that 80$\%$ of data in $\mathcal{D}_{\text{off}}$ are smaller than this threshold. 
% The results, presented in Figure~\ref{fig:robu_simu_data}, shows that T-RCA is robust as its performance does not decrease significantly when varying the thresholds.

In the setting of T-DSCM, results depicting the means and variances of the F1-scores are presented in Figure~\ref{fig:robu_simu_data} (a and d), where the dashed line illustrates the performance of the method when the thresholds are correctly chosen for each time series. However, for the other two settings, true thresholds for the time series do not exist, hence there are no dashed lines in the corresponding results. The solid line illustrates the performance of the method with varying thresholds, which does not surpass the dashed line for each method. This is because, during the data generating process, the threshold is randomly chosen for each time series, potentially resulting in different thresholds in practice. However, for simplicity, we adopt the same rule to choose the threshold for each time series, which does not guarantee that the threshold for each variable is properly chosen. It is worth noting that in the setting based on T-DSCM, correctly chosen thresholds aid our proposed method in detecting root causes. As the relationship between two variables in this setting depends on thresholds, choosing thresholds higher than the correct one reduces examples for the relations where anomaly causes anomaly, thereby decreasing the performance of our method. Similarly, choosing thresholds lower than the correct one includes examples that do not accurately represent the relations where anomaly causes anomaly, also reducing the performance of our method. Overall, in this setting, the performance of our proposed methods fluctuates with changes in thresholds, but generally, the variance remains within an acceptable range. Practically, the threshold value associated with the highest performance of our proposed methods can serve as a reference for selecting the optimal threshold to monitor the time series.

For settings corresponding to Section~\ref{subsubsec:change_coefficients} and Section~\ref{subsubsec:change_noise}, results depicting the means and variances of the F1-scores are presented in Figure~\ref{fig:robu_simu_data} (b and e) and Figure~\ref{fig:robu_simu_data} (c and f), respectively. In these settings, the performance of our proposed methods remains consistent with low variance when varying the thresholds, as the causal mechanisms in these two settings are continuous.

From the results above, we can conclude that if the causal mechanism is based on thresholds, such as event-based relations, misspecification of the threshold for time series will degrade the performance of our proposed method. However, when the causal mechanism is continuous, misspecification of the threshold for time series will not significantly impact the performance of our proposed method.

\begin{figure}
	\centering
	\includegraphics[width = 0.9\textwidth, trim = 0 1.5cm 0 1cm, clip=TRUE]{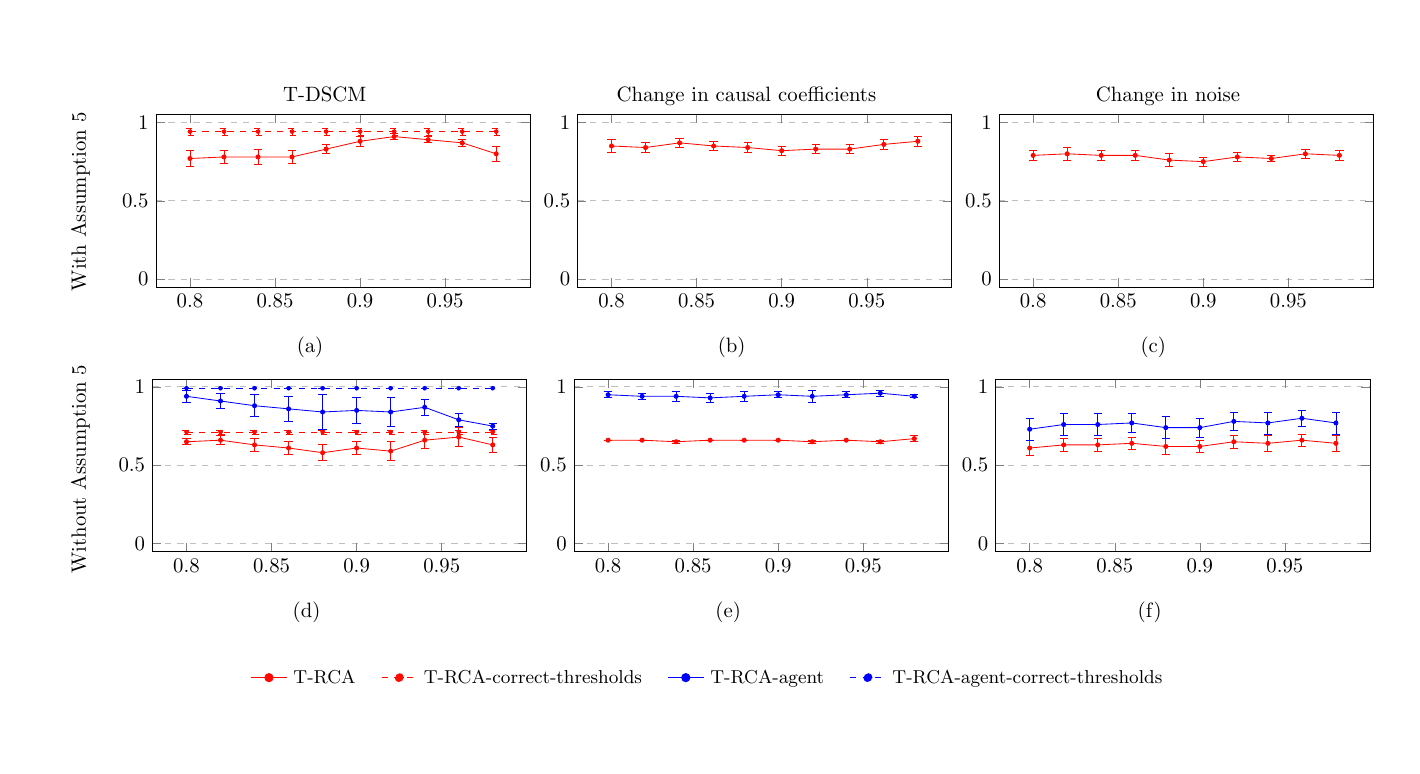}
	
	\caption{Mean of the F1-score, averaged over 50 simulations, and associated variance, for the three data generating processes of our proposed methods. We vary the threshold for each time series by controlling the proportion of data smaller than this threshold in $\mathcal{D}_{\text{off}}$ from 0.8 to 0.98 in steps of 0.02. The dashed line represents the performance of the methods when the thresholds of time series are correctly chosen}
	\label{fig:robu_simu_data}
\end{figure}

\begin{figure}[b]
	\centering
	\includegraphics[width = 0.5\textwidth, trim = 2cm 2.2cm 1.5cm 2cm, clip=TRUE]{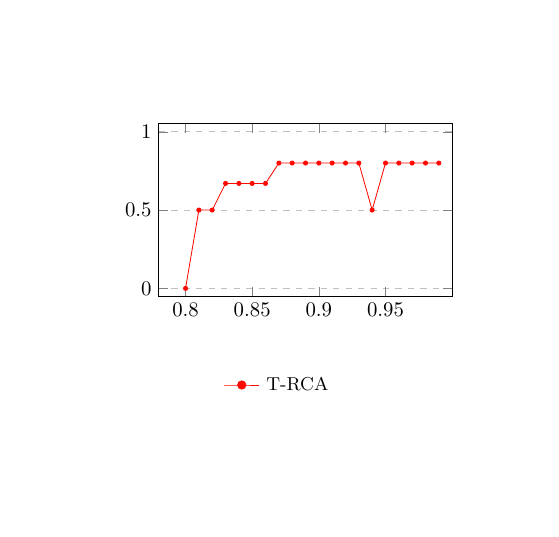}
	\caption{F1-score on Real IT monitoring data for our method, varying the threshold for each time series by controlling the proportion of data smaller than this threshold in $\mathcal{D}_{\text{off}}$ from 0.8 to 0.98 in steps of 0.01.}
	\label{fig:robu_real_data}
\end{figure}

\subsubsection{Real IT monitoring data}
\label{app:robu_monitoring_data}
Similarly, we vary the threshold for each time series by controlling the proportion of data smaller than this threshold in $\mathcal{D}_{\text{off}}$ from 0.8 to 0.98 in steps of 0.01 on real IT monitoring data, and the F1-score is shown in Figure~\ref{fig:robu_real_data}. When thresholds are chosen much lower than the correct one, examples that inaccurately represent the relations where anomaly causes anomaly are included, leading to a decline in the performance of T-RCA. Then, with the increasing of thresholds, the performance of T-RCA tends to increase. Except, at the point where the proportion is 0.94, the performance of the method exhibits significant variance compared to nearby points, as we adopt the same rule to choose the threshold for each time series for simplicity. However, in practice, the correct threshold for each time series may vary. The result also confirms, on the other hand, that our system is threshold-based. 
% Note that that by fine-tuning thresholds for each time series, T-RCA can identify the root causes with an F1-score equals to 1.

\subsection{Robustness according to the choice of the threshold - different thresholds}\label{app:rob2} 

\subsubsection{Simulated data}
We also selected an alternative strategy to test the robustness of our method by varying the thresholds. Specifically, we vary the thresholds for each time series by offsetting them from the correct thresholds, ranging from -1 to 1 in steps of 0.01.

The results given in Figure~\ref{fig:different_thres_simu_data} which shows that when the specified threshold is a higher than the true threshold then T-RCA (and T-RCA-agent) manage to keep a good performance. This means that when a system expert is hesitating between two different values (but not significantly different) for the threshold, it is better to choose the higher one.

\begin{figure}[!ht]
	\centering
	\includegraphics[width= 0.5\textwidth, trim = 1.cm 0.5cm 0.5cm 0.5cm, clip=TRUE]{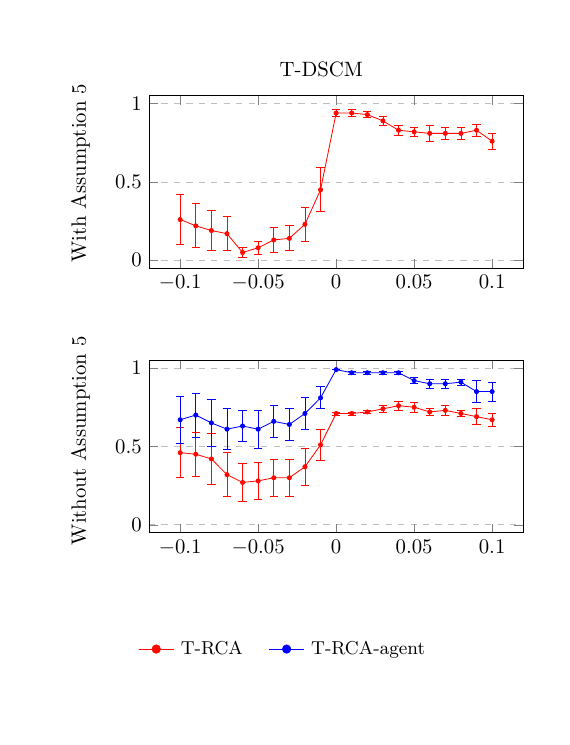}
	\caption{Mean of the F1-score, averaged over 50 simulations, and associated variance, for the data generating processes based on T-DSCM of our proposed methods. We varied the thresholds for each time series by offsetting them from the correct thresholds, ranging from -1 to 1 in steps of 0.01.}
	\label{fig:different_thres_simu_data}
\end{figure}

\subsubsection{Real IT monitoring data}\label{app:tuning} 
Since we do not have the true thresholds for each time series, we start here by fine-tuning the thresholds that gives the best results then we perform the same analysis that was done for simulated data.

\textbf{Fine-tuning thresholds}
To fine-tune the thresholds we tested randomly several threshold and selected the one that gives the highest F1-score.

% {'Check message bolt': [0.29387733641856273], 'Real time merger bolt': [0.002185499421519475], 'Message dispatcher bolt': [0.4606770536058619], 'Metric bolt': [0.6291697647512534], 'Pre-Message dispatcher bolt': [-0.06832047821056697], 'Last_metric_bolt': [0.7776513176500836], 'Elastic_search_bolt': [-0.06579879161845996], 'Group status information bolt': [0.5953778120581051]}

Surprisingly, we were able to find several sets of thresholds for which T-RCA can identify the root causes with an F1-score equals to 1. This might confirms, that our system is threshold-based.

\textbf{Varying thresholds with respect to the fine-tuned thresholds} Here we choose the thresholds that correspond to the smallest values for which T-RCA can identify the root causes with an F1-score equals to 1.

As before, we vary the thresholds for all time series by iteratively decreasing the fine-tuned thresholds by 0.01 or increasing them by 0.01, repeated five times on each side. However, an exception is made for the time series RTMB, where we decrease and increase the thresholds by 0.0001. This exception is warranted as all values of this time series are smaller than 0.01.

\begin{figure}[!ht]
	\centering
	\includegraphics[width = 0.5\textwidth, trim = 2cm 2.2cm 1.5cm 2cm, clip=TRUE]{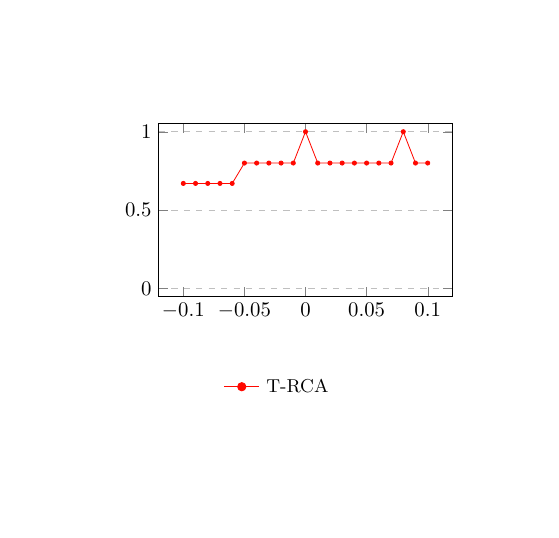}
	\caption{F1-score on Real IT monitoring data for our method, varying the threshold for each time series by offsetting them from the fine-tuning thresholds. We varied the thresholds from -1 to 1 in steps of 0.01. Specifically, for RTMB, due to its small values, we adjusted its threshold from -0.001 to 0.001 in steps of 0.0001.}
	\label{fig:different_thres_real_data}
\end{figure}

The results given in Figure~\ref{fig:different_thres_real_data}, again which again shows that if a system expert is undecided between two slightly different threshold values, it is better to opt for the higher one.

\end{document}